\documentclass[opre,sglanonrev,11pt]{informs4}
\RequirePackage{tgtermes}
\RequirePackage{newtxtext}
\RequirePackage{newtxmath}
\RequirePackage{bm}
\RequirePackage{endnotes}

\OneAndAHalfSpacedXII 

\usepackage{algorithm}
\usepackage{algpseudocode}
\usepackage{tikz}

\usepackage{natbib}
 \bibpunct[, ]{(}{)}{,}{a}{}{,}%
 \def\bibfont{\small}%

\usepackage[pdftex,colorlinks=true,urlcolor=blue,citecolor=black,anchorcolor=black,linkcolor=black]{hyperref}
\usepackage[utf8]{inputenc}

\usepackage{graphicx}
\graphicspath{ {./figures/} }
\usepackage{subcaption}

\usepackage{lineno}

\usepackage{enumitem}
\usepackage{multirow}
\usepackage{makecell}
\usepackage{float}

\EquationsNumberedThrough    

\TheoremsNumberedThrough     
\ECRepeatTheorems  %

\MANUSCRIPTNO{MOOR-0001-2024.00}

\begin{document}


\RUNAUTHOR{X}

\RUNTITLE{Dual-Agent Deep Reinforcement Learning for Dynamic Pricing and Replenishment}

\TITLE{Dual-Agent Deep Reinforcement Learning for Dynamic Pricing and Replenishment}

\ARTICLEAUTHORS{%
\AUTHOR{Yi Zheng}
\AFF{Guanghua School of Management, Peking University, \EMAIL{yizheng@stu.pku.edu.cn}}

\AUTHOR{Zehao Li}
\AFF{Guanghua School of Management, Peking University, \EMAIL{zehaoli@stu.pku.edu.cn}}

\AUTHOR{Peng Jiang}
\AFF{School of Management, Huazhong University of Science and Technology, \EMAIL{pengjiang0731@outlook.com}}

\AUTHOR{Yijie Peng}
\AFF{Guanghua School of Management, Peking University, \EMAIL{pengyijie@pku.edu.cn}}


} 
\ABSTRACT{%
We study the dynamic pricing and replenishment problems under inconsistent decision frequencies. Different from the traditional demand assumption, the discreteness of demand and the parameter within the Poisson distribution as a function of price introduce complexity into analyzing the problem property.  We demonstrate the concavity of the single-period profit function with respect to product price and inventory within their respective domains. The demand model is enhanced by integrating a decision tree-based machine learning approach, trained on  comprehensive market data. Employing a two-timescale stochastic approximation scheme, we address the discrepancies in decision frequencies between pricing and replenishment, ensuring convergence to local optimum. We further refine our methodology by incorporating deep reinforcement learning (DRL) techniques and propose a fast-slow dual-agent DRL algorithm. In this approach, two agents handle pricing and inventory and are updated on different scales. Numerical results from both single and multiple products scenarios validate the effectiveness of our methods. 
}%

\FUNDING{This research was supported by [grant number, funding agency].}




\maketitle
\section{Introduction}
Dynamic pricing refers to the practice where sellers dynamically adjust prices based on factors such as sales period, historical demand information, and supply availability \citep{elmaghraby2003dynamic}, which typically faces challenges of high costs or low feasibility in practice. Dynamic pricing methods have been applied to industries where capacity (supply) is difficult to change in the short term, such as airlines, cruise lines, hotels, and power facilities \citep{gallego1994optimal,gallego1997multiproduct}. Sectors like retail, with more flexible replenishment or higher price-change costs, have traditionally focused on enhancing inventory management.  Appropriate coordination between pricing and inventory decisions can mitigate inventory-demand mismatches \citep{feng2020integrating} and optimize profits \citep{lei2018joint}. \par

The rapid development of e-commerce and information systems has reduced the costs associated with price adjustments and enhanced the ability for real-time inventory monitoring \citep{brynjolfsson2000frictionless}. The challenge lies in that retailers often need to coordinate the pricing and inventory of thousands of products simultaneously. The existing literature suffers the following limitations: (i) The optimality of the base-stock list-price (BSLP) and (s,S,p) policies requires assuming restrictive  functional forms of demand, such as additive or multiplicative demands. (ii) Existing literature often assumes that excess demand is backlogged. For retail firms selling to consumers (i.e., B2C), it is more common for consumers to take their demand elsewhere during stock-outs, leading to lost-sales \citep{bijvank2023lost}. 
(iii) Most previous works consider demand in a certain period that depends only on prices, without considering the impact of other market factors, such as competitors. We take into consideration the impact of market competition (including competitors' prices, market reference prices, and price rankings) while optimizing pricing and replenishment strategies. \par

In the literature, strong assumptions about market demand have often been made. For example, \citet{cohen2018dynamic} assume that the demand is a linear function of the price plus a random noise component, and \citet{aviv2005partially} assume that customer arrivals follow a Poisson distribution. In reality, demand is not solely related to price, and the relationship is not as simple as a linear one.
Our work is not constrained by the unrealistic assumptions mentioned above. \par

A recent trend gaining steam is to introduce reinforcement learning (RL) to solve sequential inventory decision problems \citep{gijsbrechts2021can,oroojlooyjadid2022deep}. Through continuous trial-and-error interactions
with an unknown demand, RL agents learn by themselves to achieve optimal strategies without
imposing assumptions on the demand \citep{sutton1999reinforcement}. Deep RL (DRL) is a subfield of RL, which employs deep neural networks to approximate the value or policy functions of Markov Decision Processes (MDPs). \citet{rana2014real} investigate the pricing problem for a given inventory within a fixed time period.  \citet{kastius2021dynamic} apply Deep Q-Networks (DQN) and Soft Actor Critic (SAC) RL algorithms to the pricing problem of a single product in duopoly and oligopoly markets where the customer behavior is described by the logit model in \citet{schlosser2018dynamic}.
\par
We first investigate a dynamic pricing and replenishment problem under competition, considering competitor strategies and market conditions. Different from the traditional demand assumption, the discreteness of demand we consider, along with the parameter within the Poisson distribution being a function of price, introduces complexity into analyzing the property of the problem. We prove that the single-period profit function is concave with respect to product price and inventory within respective domains, but it is not jointly concave. Based on the respective concavity, a two-timescale algorithm is proposed. For multi-period problems, we propose a dual-agent DRL algorithm combining the two-timescale approach with Multi-Agent DRL (MADRL). Two agents are introduced in the algorithm, responsible for pricing and inventory decisions respectively. This allows each agent to optimize its own decisions while fully considering the influence of the other agent. Furthermore, we allow the two agents for making decisions on pricing and inventory to be updated in different scales. Finally, we explore a dynamic pricing and replenishment problem where product demand is directly learned from comprehensive market data, encompassing product attributes, customer behavior, and marketing information. The simulation results indicate that our joint pricing and replenishment strategy, which incorporates demand forecasting and further price optimization, results in lower inventory and higher profits. \par
Our work consists of the following contributions.\par
(i) We prove that the single-period profit function in dynamic pricing and replenishment problem is concave with respect to product price and inventory within respective domains, but it is not jointly concave. The respective concavity ensures the convergence of our proposed two-timescale algorithm, with solutions converging to the limit points of the ordinary differential equation (ODE).\par

(ii) We formulate the dynamic pricing and replenishment problems as Markov Games (MGs) and propose a new solution integrating two-timescale algorithms with MADRL to construct intelligent pricing and ordering policies. Additionally, we enhance our model's capability to capture real market demand by combining a decision tree-based machine learning (ML) model with real market data to train the ML model.\par

(iii) We experimentally demonstrate that our algorithms achieve optimal performance on solvable numerical examples and outperform baseline algorithms in cases where optimal strategies are not analytical, thus validating their effectiveness. Furthermore, our method exhibits linear computational complexity relative to problem scale, ensuring efficiency in practical pricing and replenishment applications.

\section{Literature Review}
In this section, a brief review on related literature is provided. 
We begin by reviewing the literature on joint dynamic pricing and inventory control, and then proceed to review the problem of dynamic pricing and replenishment under competition. Relevant computing approach and their applications in the field of dynamic pricing and inventory management are introduced.

\subsection{Joint Dynamic Pricing and Inventory Control}
The joint pricing and inventory control has received considerable attention over the past decade, often formulated as a Markov Decision Process (MDP) for multi-period decision problems in finite or infinite horizons \citep{federgruen1999combined}. To solve this problem, strong assumptions have been made in the literature. For instance, \citet{jia2011dynamic} assume the demand follows a random price-dependent function; \citet{bisi2007dynamic} study an additive and multiplicative price-sensitive demand; \citet{raz2006fractiles} approximate a continuous demand model using a finite number of representative fractiles, which are piecewise linear functions of the price. \citet{federgruen1999combined} assume that replenishment, holding, and ordering costs are convex, and the seller has unlimited production capacity. \citet{chen2004coordinating} and \citet{chan2006pricing} extend the previous models by considering fixed ordering costs  and limited production capacity, respectively. \par
 
\citet{thowsen1975dynamic} proves that a BSLP policy is optimal for the additive demand function by assuming the random noise follows a Pólya frequency function of order 2 (PF2) distribution. 
\citet{federgruen1999combined} further establish the optimality of a BSLP policy with a general demand function for a backorder system without capacity constraints. Their sufficient conditions are as
follows: (i) the demand is decreasing and concave in the list price; (ii) the single-period expected
inventory cost is jointly convex in the base-stock level and price, which can be guaranteed only if the demand function is linear in the price \citep{feng2014dynamic}. With fixed ordering costs, \citet{chen2004coordinating}  prove that the $(s,S,p)$ policy is optimal when the demand is additive in a finite time horizon and the seller’s objective is to either maximize expected discounted profits or maximize average long-run profits in the infinite-horizon model.  \par

\subsection{Dynamic Pricing and Replenishment under Competition}

\citet{chen2012pricing} note that research on jointly optimizing pricing and inventory decisions under competition remains limited due to model complexity. Competition typically drives prices lower, reducing firms' profits compared to monopolistic scenarios \citep{mantin2011dynamic,mookherjee2008pricing}, and may lead to excessive inventory holdings \citep{xu2006monopolistic}. \citet{mahmoodi2019joint} examines the pricing and replenishment problem of a single product with deterministic linear demand in a duopoly market and also compares the case where two substitutable products are both sold by monopolists but does not consider products and competition simultaneously.  
\citet{schlosser2018dynamic} analyze stochastic dynamic pricing models in competitive markets with multiple offer
dimensions (e.g., product qualities and seller ratings). They compare different demand learning techniques, including least squares, gradient boosted trees,  and logistic regression, to estimate sales probabilities from partially observable market data. Ultimately, they adopt a logistic regression approach for its interpretability and performance. In the follow-up study, they further consider market reference prices, which are related to historical market prices, as well as competitor strategies and price rankings \citep{schlosser2019dynamic}. 

\subsection{Two-timescale Stochastic Approximation and Deep Reinforcement Learning}
Two-timescale algorithms aggregate data and update parameters at epochs with two timescales, which allows for discrete-valued observations and avoids the infrequent aggregation of data over regeneration epochs \citep{bhatnagar1998two}, and can handle the frequency discrepancy between different decision-making processes \citep{wernz2013multi}. \par
MDPs provide a mathematical framework to model decision-making problems in sequential environments. RL learns optimal policies or value functions in MDPs through iterative interactions with the environment, basing decisions on observed feedback. RL algorithms can be divided into three categories: value-based, policy-based, and actor-critic algorithms \citep{arulkumaran2017deep}. Value-based algorithms focus on learning the value function, which estimates the expected cumulative reward for actions taken in specific states. Policy-based algorithms directly learn the optimal policy without explicitly estimating the value function. Actor-critic algorithms combine value-based and policy-based approaches, utilizing an actor to learn the policy and a critic to estimate the value function. \par

While RL has shown promise in solving complex decision-making problems, it faces challenges with high-dimensional inputs, scalability, and generalization, which could be addressed by DRL algorithms, such as Asynchronous Advantage Actor-Critic (A3C) \citep{mnih2016asynchronous} and Proximal Policy Optimization (PPO) \citep{schulman2017proximal} algorithms. A3C employs multiple agents running in parallel, enabling effective exploration and faster convergence, and PPO ensures monotonic policy improvement in each training iteration based on the trust region learning proposed by \citet{schulman2015trust}. These algorithms have demonstrated excellent performance in supply chain management. For instance, \citet{gijsbrechts2021can} apply A3C to three classic inventory problems and their result shows that A3C beats a Base Stock policy by 9\%-12\%. 
\citet{vanvuchelen2020use} apply the PPO algorithm to the joint replenishment problem for multiple items and they show that PPO outperforms two other heuristics.\par 

In Multi-Agent Systems (MAS), multiple agents cooperate or compete within a shared environment to achieve common or individual goals. For any individual agent in a MAS, the transition of environmental states is no longer solely dependent on its own actions and the dynamics of the environment; it also depends on the actions taken by other agents, leading to non-stationarity in the environment. This is particularly evident in the dynamic pricing and replenishment problem investigated in this study, where the local objectives of different decisions are conflicting: for example, the pricing strategy aims to increase profits through high pricing, while inventory management aims to quickly clear stock through low pricing. A single-agent DRL algorithm can not effectively coordinate the competition between them, resulting in sub-optimal or unstable strategies. In Section 4 of our experiments, we apply a single-agent DRL algorithm, but the results were unsatisfactory, with a convergence speed merely 1/5th of the algorithms we subsequently adopted. \par
To solve MDPs in MAS, MADRL has been developed. The Centralized Training and Decentralized Execution (CTDE) scheme, a mainstream MADRL architecture, addresses non-stationarity in MAS environments \citep{foerster2016learning}. It trains agents with full environmental information but executes actions based only on local information. This architecture enables global optimization during training by sharing full information and decentralized execution during testing, ensuring both agent cooperation and system efficiency.  By applying CTDE to Deep Deterministic Policy Gradient (DDPG), \citet{lowe2017multi} show that Multi-Agent DDPG (MADDPG) has a desirable performance under both cooperative and competitive multi-agent environments. Similarly, \citet{yu2022surprising} apply CTDE to PPO and introduce Multi-Agent PPO (MAPPO). Unlike the PPO algorithm, MAPPO does not guarantee monotonic improvement. To address these issues, \citet{kuba2021trust} propose Heterogeneous-Agent Proximal Policy Optimization (HAPPO). 

\section{Dynamic Pricing and Replenishment under Competition}
In this section, we formulate the problem of dynamic pricing and replenishment under competition as an MDP and illustrate the dynamics of this problem. We first derive theoretical results for the single-period problem and propose a two-timescale approach based on these findings. Then, considering the multi-period problem structure, we combine this approach with MADRL to develop a novel fast-slow dual-agent DRL algorithm for solving multi-period dynamic pricing and replenishment problems. Finally, we demonstrate the convergence and computational complexity of our algorithm, which significantly outperforms other heuristic strategies in experimental results.

\subsection{Markov Decision Process Formulation}
An MDP can be defined as a tuple $<\mathcal{S},\mathcal{A},\mathcal{P},\mathcal{R},\gamma>$. $\mathcal{S}$ is the set of product states. The state $s_t$ in period $t$ encompasses: the previous period's ending inventory level, lost-sales, demand, price, and ordering information within the lead time. $\mathcal{A}$ is the set of product actions $a_t$. Here, action $a_t$ in period $t$ determines price $p_t$ and replenishment quantity $q_t$. $\mathcal{P}: \mathcal{S} \times \mathcal{A} \times \mathcal{S} \rightarrow [0,1]$  is the state transition function that maps the current  state, action and next possible state to a transition probability; $\mathcal{R}: \mathcal{S} \times \mathcal{A} \rightarrow \mathbb{R}$ represents the reward function that maps current state and action to a payoff and $\gamma \in (0,1]$ is the discount factor. In each period $t$, the agent receives $s_t \in \mathcal{S}$ and takes an action $a_t \sim \pi(\cdot | s_t)$, where $\pi$ is a policy which maps states to action’s probability distribution. After taking action $a_t$, a reward $\mathcal{R}(a_t|s_t) = r_t$ set as the sales profit in period $t$ is observed. Meanwhile, the system transfers to the next state $s_{t+1} \sim \mathcal{P}(\cdot|a_t,s_t)$, where the state transition function $\mathcal{P}$ is determined by the dynamics of the pricing and replenishment model.

At every time period, the decision maker decides the price and replenishment quantity of a product for this period in order to maximize the discounted future profits, which are defined as
\begin{equation} 
    \begin{aligned} \label{eq:8}
   R_t = \sum_{k=0}^{T-t} \gamma^{k} r_{t+k}.
    \end{aligned}
\end{equation}

The expected value of these discounted profits $R_t$  corresponds to the value function $V^{\pi_t}(s_t)$, which indicates the value of being in state $s_t$ following a certain policy $\pi_t$:

\begin{equation} 
    \begin{aligned} \label{eq:9}
   V^{\pi_t}(s_t) = \sum_{k=0}^{T-t} \gamma^{k} \mathbb{E}_{\pi_t} r_{t+k}.
    \end{aligned}
\end{equation}

The value $V(s_t)$ is recursively dependent on the value of the next state $V(s_{t+1})$. The optimal pricing and replenishment action in period $t$ can be determined by solving the well-known Bellman equations:

\begin{equation} 
    \begin{aligned} \label{eq:10}
   V^{\pi_*}(s_t) = \max \limits_{a_t \in \mathcal{A}} \Big\{ r_t + \gamma \sum_{s_{t+1}} \mathcal{P}(s_{t+1}|s_t,a_t) V^{\pi_*}(s_{t+1}) \Big\}.
    \end{aligned}
\end{equation}

Traditional demand assumptions mainly follow the form: 
\begin{equation}\label{eq:new10}
    d_{t}(p_t,\epsilon_t) = \gamma_t(p_t)\epsilon_t + \delta_t(p_t),
\end{equation}
with $\gamma(\cdot)$ and $\delta(\cdot)$ being nonincreasing functions. The cases of $\gamma_t(p)=1$ and $\delta_t(p)=0$ are often referred to as the additive and multiplicative models, respectively.
Continuous demand models are often preferred due to their mathematical tractability and computational efficiency \citep{babai2011analysis}. However, real-world demands could be discrete. Our analysis of real market demand data, using four common continuous demand models, results in poor fitness with R-squared values ranging only between 0.106 and 0.146. Example \ref{example:1} in Appendix B provides further illustration of this inadequacy. \citet{swaminathan1999supplier} point out that optimal ordering policies for discrete and continuous demand distributions may be totally different. In Section 3.2, we formulate a problem with the discrete Poisson demand model.

\subsection{Dynamics of the Dynamic Pricing and Replenishment Model}
Consider a $T$-period problem, where Table \ref{tab:parameters and varibles1} presents the notations. At the beginning of each period $t$ $(t = 1,...,T)$, product prices and replenishment quantities need to be determined. Product demand is uncertain, influenced not only by its own price but also by competitors' prices and market conditions. 
We infer overall demand by modeling the probability of each individual consumer's choice. Specifically, customer behavior is captured through a logistic model, which is based on real-world data \citep{schlosser2019dynamic}. Given our offer price $p$, the competitor’s price $o$, and the current reference price $j$, the sales probability within a time span of length $\Delta \in [0,1]$ is denoted by
\begin{equation} \label{eq:1}
    \begin{aligned}
    P^{\Delta}(p,o,j):=
     Pois\Big(\Delta \cdot e^{\bold{\kappa}(p,o,j)^{'}\bold{\beta}} \text{/} (1+e^{\bold{\kappa}(p,o,j)^{'}\bold{\beta}})\Big),
    \end{aligned}
\end{equation}
 where $\bold{\kappa}(p,o,j)$ represents regressors relating to sales, such as the competitor's price and the reference price. $\bold{\beta}$ represents the corresponding parameter vector. This model calculates the probability of a sale to an individual customer in a given period based on these price-related features. The demand for a product in period $t$ follows
\begin{equation} \label{eq:2}
    d_t  \sim Pois\Big( \eta \Delta \cdot e^{\bold{\kappa}(p,o,j)^{'}\bold{\beta}} \text{/} (1+e^{\bold{\kappa}(p,o,j)^{'}\bold{\beta}}) \Big).
\end{equation}
The scaling factor $\eta$ is introduced to adjust the magnitude of the Poisson distribution parameter $\lambda$.
\begin{table}[htb]
\renewcommand\arraystretch{0.8}
\centering
\caption{Parameters and variables \label{tab:parameters and varibles1}}{
\begin{tabular}{cl}\hline
\multicolumn{2}{l}{Parameters}  \\  \hline
$h$  & unit holding cost    \\ \hline
$b$ & unit shortage cost  \\ \hline
$c$ & unit ordering cost   \\ \hline
$z$ & lead time   \\ \hline
\multicolumn{2}{l}{Variables} \\ \hline
$q_t$ & replenishment quantity offered by agents in period $t$\\ \hline
$p_t$ & price offered by the DRL agent in period $t$ \\ \hline
$o_t$ & price offered by the competitor in period $t$ \\ \hline
$d_t$ & demand in period $t$  \\ \hline
$I_t$ & inventory level in period $t$ \\ \hline
$L_t$ & lost-sales in period $t$ \\ \hline
$S_t$ & actual sales in period $t$ \\ \hline
\end{tabular}}
\end{table}

We assume that excess demand is lost and unsold products have no salvage value. This lost-sales assumption complicates the model compared to backlog scenarios \citep{morton1971near}. \citet{arrow1958studies} demonstrate that in a single sourcing model with backlogging and constant lead time, inventory and outstanding orders can be simplified to one inventory position. When excess demand is lost, this simplification does not apply, and the optimal policy relies on the full inventory pipeline vector, exponentially increasing complexity with lead time. \par
At the beginning of any period $t$, the replenishment quantity $q_t$ and price $p_t$ for a product must be decided. Since the unmet demand is lost, the dynamics of the evolution of inventory can be expressed by \par
\begin{equation} \label{eq:3}
    I_t = [I_{t-1}+q_{t-{z}}-d_t]^+.
\end{equation}
The lost-sales is 
\begin{equation}\label{eq:5}
    L_t = \max \{d_t-(I_{t-1}+q_{t-{z}}),0\}.
\end{equation}
The actual sales is given by
\begin{equation}\label{eq:4}
    S_t = \min \{d_t,I_{t-1}+q_{t-{z}}\}.
\end{equation}
The sales profit $r_t$ is
\begin{equation}\label{eq:6}
    r_t = p_t S_t - h I_t - b L_t - c q_t.
\end{equation}
\par

\subsection{Theoretical Analysis for Single-period Problem}
Before studying multi-period decision problems, we first consider a single-period problem with zero lead time and ending inventory given by
\begin{equation}
    I = [x_0+q-d]^+,
\end{equation}
where $q$ represents the replenishment quantity, $x_0$ is the initial inventory at the beginning of the period, which is known, and $d$ is the demand for the current period.\par
The lost-sales is $L = \max \{d-(x_0+q),0\}$.
The actual sales is given by $S = \min \{d,x_0+q\}$. The objective is to find the optimal order quantity $q^*$ and the optimal price $p^*$, in order to maximize the single-period expected profit depending on the demand model, as follows:
\begin{equation} 
    \begin{aligned} \nonumber
F(p, x) = \mathbb{E}[pS(p,x) - hI(p,x) - bL(p,x) - c(x-x_0)],
    \end{aligned}
\end{equation}
where $x=x_0+q$. Therefore, the problem comes down to deciding on $p$ and $x$, given that $x_0$ is known.\par 
Unlike the traditional demand assumption in Eq.(\ref{eq:new10}), our discrete Poisson demand model (Eq.(\ref{eq:2})) incorporates a price-dependent parameter $\lambda(p)$. This dependency introduces complexity into the analysis of related functions and deviates from conventional assumptions.

\begin{proposition}\label{pro:necessity}
    The Poisson distribution $D(p) \sim Pois(\lambda(p))$, where $\lambda(p)$ is the parameter, does not fall under the traditional demand assumption described in Eq.(\ref{eq:new10}): $d(p,\epsilon) = \gamma(p)\epsilon + \delta(p)$, with $\epsilon$ being a random variable that is independent of $p$.
\end{proposition}
The proof of Proposition \ref{pro:necessity}, along with the following omitted proofs, can be found in Appendix A. This proposition highlights the structural difference between our Poisson demand model $D(p) \sim Pois(\lambda(p))$ and traditional models. The key distinction lies in the inseparability of the price-dependent function $\lambda(p)$ from the distribution of the random variable in the Poisson model, unlike the explicit separation of the random variable $\epsilon$ and price-dependent functions in the traditional model $\gamma(p)\epsilon + \delta(p)$.\par
Furthermore, previous studies on joint pricing and replenishment often focus on scenarios with demand backlog, as in \cite{federgruen1999combined}, where revenue depends on demand volume rather than sales, making the revenue function independent of inventory. In contrast, we analyze scenarios with lost-sales under discrete demand (Eq.(\ref{eq:2})). We derive the concavity and convexity of the inventory function, revenue function, and lost-sales function under discrete demand Eq.(\ref{eq:2}). The final comparison with traditional assumptions is presented in Proposition \ref{thrm:2}. 
\subsubsection{Properties of Objective Function}
\begin{assumption}\label{assumption:1}
    We suppose $lp+a > 0$ in the domain of $p$ and the domain is a bounded set.
\end{assumption}

\begin{lemma}\label{lemma:1}
    Under Assumption \ref{assumption:1}, define $$\lambda = \eta \Delta \cdot e^{lp + a} \text{/} (1+e^{lp + a}),$$ then $\lambda(p)$ is concave and monotonically decreasing with respect to $p$.
\end{lemma}

\begin{proposition}\label{proposition:2}
Under Assumption \ref{assumption:1}, the inventory function $I(p,x)=\mathbb{E}[(x-d)^+]$ is convex and monotonically increasing with respect to price $p$.
\end{proposition}
\begin{proof}{Proof}
Denote $\lfloor x \rfloor$ as the maximum integer no more than $x$. We rewrite $I(p,x)$ as
    \begin{equation*} \small
        \begin{aligned}
            \mathbb{E}[(x-d)^+] = \sum_{k=0}^{\lfloor x \rfloor}\frac{\lambda^k e^{-\lambda}}{k!}(x-k).
        \end{aligned}
    \end{equation*}
    Then we calculate the first order derivative and the second order derivative:
    \begin{equation*} \small
        \begin{aligned}
        \frac{\partial I(p,x)}{\partial p} = \frac{\partial \lambda}{\partial p }\sum_{k=0}^{\lfloor x \rfloor}\frac{\lambda^{k-1} e^{-\lambda}}{k!}(x-k)(k-\lambda),
        \end{aligned}
    \end{equation*}
     \begin{equation}\small
        \begin{aligned}\label{eq:11}
    \frac{\partial^2 I(p,x)}{\partial p^2} = \frac{\partial^2 \lambda}{\partial p^2}\sum_{k=0}^{\lfloor x \rfloor}\frac{\lambda^{k-1} e^{-\lambda}}{k!}(x-k)(k-\lambda) + (\frac{\partial \lambda}{\partial p })^2\sum_{k=0}^{\lfloor x \rfloor}\frac{\lambda^{k-2} e^{-\lambda}}{k!}(x-k)((k-\lambda)^2-k).
     \end{aligned}
    \end{equation}
    We want to prove $\frac{\partial^2 I(p,x)}{\partial p^2}>0$. First, we consider the first term of Eq.(\ref{eq:11}). 
   When $x<\lambda$, for every $ k \le \lfloor x \rfloor$, $k-\lambda<0$ holds, we have $\sum_{k=0}^{\lfloor x \rfloor}\frac{\lambda^{k-1} e^{-\lambda}}{k!}(x-k)(k-\lambda)<0$. When $x \ge \lambda$, since $(x-k)(k-\lambda) \le (x-\lambda)(k-\lambda)$,  we have
    \begin{equation*}\small
        \begin{aligned}
            \sum_{k=0}^{\lfloor x \rfloor}\frac{\lambda^{k-1} e^{-\lambda}}{k!}(x-k)(k-\lambda) &\le (x-\lambda)\sum_{k=0}^{\lfloor x \rfloor}\frac{\lambda^{k-1} e^{-\lambda}}{k!}(k-\lambda) \\ = \frac{x-\lambda}{\lambda}\sum_{k=0}^{\lfloor x \rfloor}\mathbb{P}(X=k)(k-\lambda)  &= \frac{x-\lambda}{\lambda}(\sum_{k=0}^{\lfloor x \rfloor}k \mathbb{P}(X=k) -\lambda \mathbb{P}(X \le \lfloor x \rfloor),
        \end{aligned}
    \end{equation*}
where $X$ is a random variable following a Poisson distribution with mean $\lambda$. 
\vspace{-0.3cm}
\begin{figure}[H]
  \centering
  \caption{Poisson distribution and its truncated version}
  \includegraphics[width=0.65\textwidth]{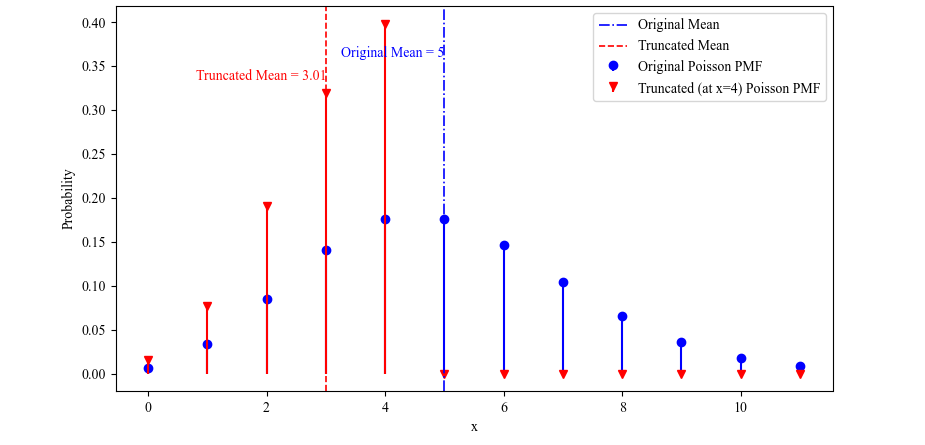} 
  \label{fig:truncated}
\end{figure}  
\vspace{-1.0cm}
Considering a random variable $X'$ following Poisson distribution truncating at $\lfloor x \rfloor$, it is obvious that $E[X'] <  E[X]$ due to the truncation, as shown in Figure \ref{fig:truncated}. Therefore, \begin{equation*}\small
        \begin{aligned}
    E[X'] = \frac{\sum_{k=0}^{\lfloor x \rfloor}k \mathbb{P}(X=k)}{\mathbb{P}(X \le \lfloor x \rfloor)} <  E[X] = \lambda,        \end{aligned}
\end{equation*}
        which implies $\sum_{k=0}^{\lfloor x \rfloor}\frac{\lambda^{k-1} e^{-\lambda}}{k!}(x-k)(k-\lambda)<0$.
    Note that $\frac{\partial^2 \lambda}{\partial p^2}<0$, we prove that the first term of Eq.(\ref{eq:11}) is positive. The conclusion $\frac{\partial I(p,x)}{\partial p} > 0$ holds because $\frac{\partial \lambda}{\partial p }<0$. Next, we claim that the second term of Eq.(\ref{eq:11}) is also positive. Suppose that $X$ is a random variable following a Poisson distribution with mean $\lambda$, then $\mathbb{E}[X^2] = \lambda^2+\lambda$, $\mathbb{E}[X^3] = \lambda^3 + 3\lambda^2 + \lambda$. We can find that
    \begin{equation}\small
        \begin{aligned}\notag
            \mathbb{E}[(x-X)((X-\lambda)^2-X)] =& -\mathbb{E}[X^3] + (x+(2\lambda+1))\mathbb{E}[X^2] + (-\lambda^2-(2\lambda+1)x)\mathbb{E}X + \lambda^2x \\
            =& -\lambda^3 - 3\lambda^2 - \lambda + (x+(2\lambda+1))(\lambda^2+\lambda) +(-\lambda^2-(2\lambda+1)x)\lambda +\lambda^2x  =  0.
        \end{aligned}
    \end{equation}
    Therefore, the equation below holds for every $x$:
        \begin{equation}\small
        \begin{aligned}\label{eq:12}
    \sum_{k=0}^{\infty}\frac{\lambda^{k-2} e^{-\lambda}}{k!}(x-k)((k-\lambda)^2-k) = \lambda^{-2}\mathbb{E}[(x-X)((X-\lambda)^2-X)] = 0.
        \end{aligned}
    \end{equation}
    It is obvious that the the second term of Eq.(\ref{eq:11}) is continuous and piecewise linear with respect to $x$  with integer points as nodes, which implies that this function reaches its minimum at certain integer nodes. In order to find its minimum with respect to $x$, we can calculate the difference by taking $x$ as the integers $n$ and $n + 1$, respectively, as follows: 
    \begin{equation}\small
        \begin{aligned}\label{eq:13}
           &\sum_{k=0}^{n+1}\frac{\lambda^{k-2} e^{-\lambda}}{k!}(n+1-k)((k-\lambda)^2-k) - \sum_{k=0}^{n}\frac{\lambda^{k-2} e^{-\lambda}}{k!}(n-k)((k-\lambda)^2-k)  = \sum_{k=0}^{n}\frac{\lambda^{k-2} e^{-\lambda}}{k!}((k-\lambda)^2-k).
        \end{aligned}
    \end{equation}
     \begin{figure}
\centering
\caption{Eq.(\ref{eq:13}) and the second term of Eq.(\ref{eq:11}) without quadratic term with $\lambda=5$}
\begin{subfigure}[b]{0.45\textwidth}            
   \captionsetup{font=tiny}
   \caption{Plot of Eq.(\ref{eq:13}).}
   \label{fig:illus_plot1}
   \includegraphics[width=\textwidth]{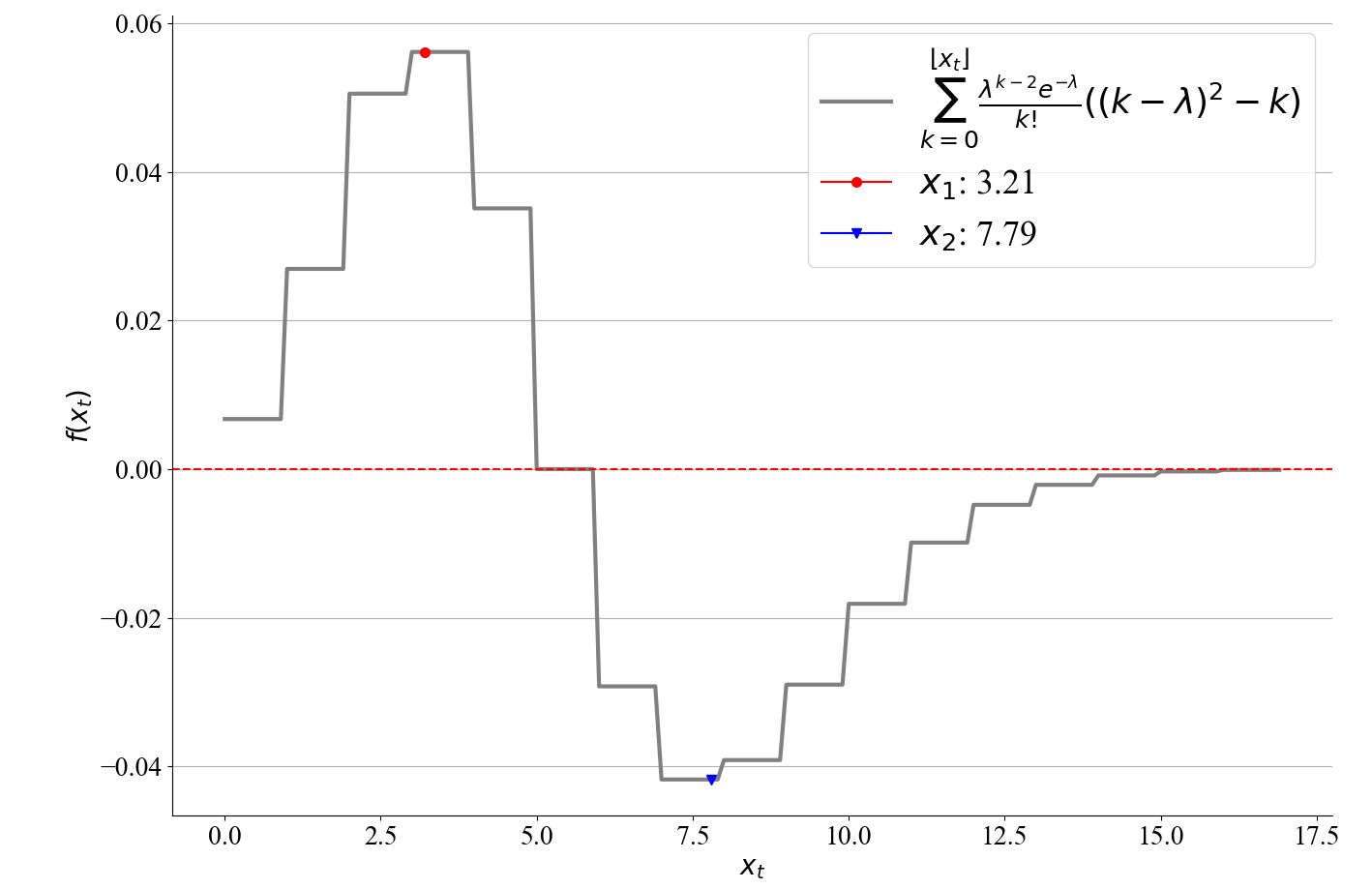}
   
\end{subfigure}
\begin{subfigure}[b]{0.45\textwidth}
   \captionsetup{font=tiny}  
   \caption{Plot of the Second Term of Eq.(\ref{eq:11}) without Quadratic Term.}
   \includegraphics[width=\textwidth]{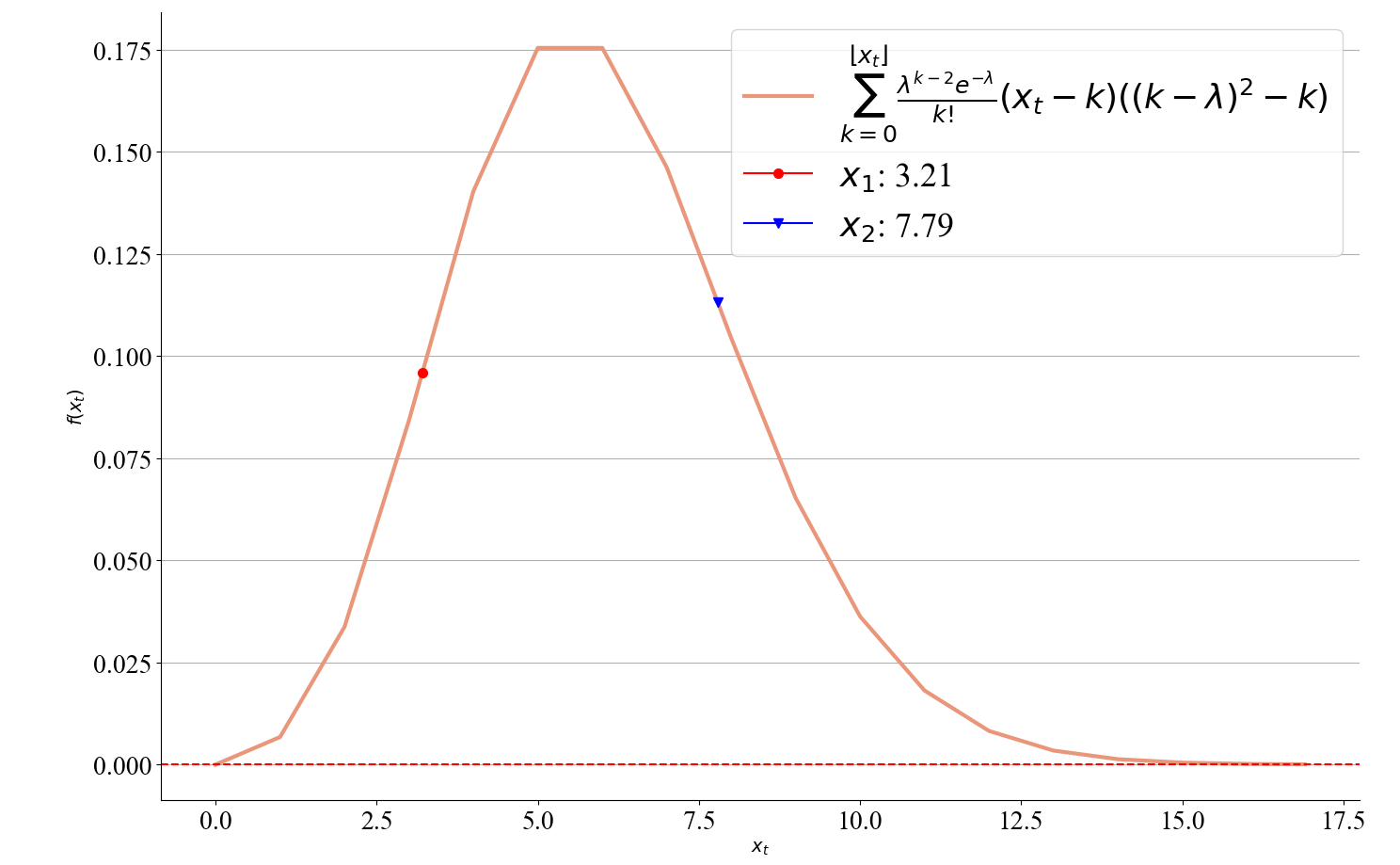}
   
   \label{fig:illus_plot2}
\end{subfigure}

\label{fig:illus_plot}
\end{figure}
    Given that $(\frac{\partial \lambda}{\partial p })^2 > 0$, the sign of Eq.(\ref{eq:13}) determines that of the derivative of the second term of Eq.(\ref{eq:11}) with respect to $x$ when $x$ is not an integer. We illustrate Eq.(\ref{eq:13}) and the second term of Eq.(\ref{eq:11}) without quadratic term $(\frac{\partial \lambda}{\partial p })^2$ with $\lambda=5$ in Figure \ref{fig:illus_plot}. The quadratic equation $(k-\lambda)^2-k$ with respect to $k$ has two roots, denoted as $x_1$ and $x_2$, where $0 < x_1 < x_2$. These roots are indicated as two red points in Figure \ref{fig:illus_plot}.  As is shown in Figure \ref{fig:illus_plot1}, when $0<x \le x_1$, the value of Eq.(\ref{eq:13}) is increasing with respect to $x$. When $x_1 < x < x_2$, the value of Eq.(\ref{eq:13}) is decreasing for $x_1<k<x_2$ as each added term containing $(k-\lambda)^2-k$ is negative. When $x > x_2$, the value of Eq.(\ref{eq:13}) is increasing for  $k>x_2$ as each added term containing $(k-\lambda)^2-k$ is positive. The minimum of the second term of Eq.(\ref{eq:11}) occurs at the point that Eq.(\ref{eq:13}) changes from negative to positive (if there exists) or at the boundary. Therefore, the second term of Eq.(\ref{eq:11}) reaches its minimum when $x=x^{*} \ge \lfloor x_2 \rfloor$ or $x = 0$.  Consequently, we have
    \begin{equation*}\small
        \begin{aligned}
           \sum_{k=0}^{\lfloor x \rfloor}\frac{\lambda^{k-2} e^{-\lambda}}{k!}(x-k)((k-\lambda)^2-k) \ge&  \min \{0, \sum_{k=0}^{x^{*}}\frac{\lambda^{k-2} e^{-\lambda}}{k!}( x^{*}-k)((k-\lambda)^2-k)\}\\ =&   \min \{0, 0 - \sum_{k=x^{*}+ 1 }^{\infty}\frac{\lambda^{k-2} e^{-\lambda}}{k!}( x^{*}-k)((k-\lambda)^2-k)\}  \ge  0.
        \end{aligned}
    \end{equation*}
    According to Eq.(\ref{eq:12}), the second equality holds. The last inequality is based on the fact that when $k\ge x^{*} + 1 \ge x_2 $, the expression $(x^{*}-k)((k-\lambda)^2-k)$ is less than 0. Therefore, the second term of Eq.(\ref{eq:11}) is positive as is shown in Figure \ref{fig:illus_plot2}. In conclusion, the two terms of the second derivative of inventory function are positive and non-negative, respectively. This indicates that the inventory function is convex in $p$. \Halmos

\end{proof}
\begin{remark}\label{remark:1}
Under Assumption \ref{assumption:1}, the inventory function $I(p,x)=\mathbb{E}[(x-d)^+]$ is neither jointly convex nor concave with respect to the inventory $x$ and price $p$ at its differentiable points. This is due to $I(p,x)$ being continuous and piecewise linear with respect to $x$, implying $\frac{\partial^2 I(p,x)}{\partial x^2} = 0$ at its differentiable points. Besides, $$\frac{\partial^2 I(p,x)}{\partial x\partial p} = \frac{\partial \lambda}{\partial p}\sum_{k=0}^{\lfloor x \rfloor}\frac{\lambda^{k-1}e^{-\lambda}}{k!}(k-\lambda) > 0,$$ which is implied by the proof of Proposition \ref{proposition:2}. Therefore, the Hessian matrix of $I(p,x)$ is indeterminate. 
\end{remark}

\begin{proposition}\label{proposition:4}
Under Assumption \ref{assumption:1}, the revenue function $E(p,x)=p\mathbb{E}[S(p,x)]=p\mathbb{E}[\min(d,x)]$ is concave with respect to price $p$.
\end{proposition}

Since the random variable $d$ follows the Poisson distribution and the parameter $\lambda$ is determined by price $p$ through a complex function $\lambda(p)$, we need more assumptions to guarantee the convexity of lost-sales function. We employ the Taylor expansion to approximate $\lambda$ and linearize it.
\begin{assumption}\label{assumption:3} 
We suppose that $\hat{\lambda} = \eta \Delta e^a(1+lp)$ with $l<0$ and $1+lp>0$ in the domain of $p$, where the domain is bounded.
\end{assumption}
\begin{lemma}\label{lemma:6}
 Under Assumption \ref{assumption:3}, the lost-sales function $L(p,x)=\mathbb{E}[(d-x)^+]$ is convex with respect to the price $p$.
 \end{lemma}

\begin{remark}\label{remark:2}
    Assumption \ref{assumption:1} aims to guarantee the concavity of function $\lambda(p)$, which can be replaced by Assumption \ref{assumption:3} because linear function is concave. Henceforth, we use $\lambda$ to stand for $\hat{\lambda}$ in Assumption \ref{assumption:3}.
\end{remark}
 
Having established the concavity or convexity of the expected inventory, revenue, and lost-sales functions, we can derive the concavity of the single-period expected profit function $F(p, x)$ with respect to $p$. 

\begin{theorem} \label{theorem:1}
     Under Assumption \ref{assumption:3}, $F(p, x)$ is concave with respect to $p$ and $x$ within their respective domains, though not jointly concave. Furthermore, $F(p,x)$ takes its maximum either at the boundary or at an interior point that satisfies the following conditions:
     
    \begin{equation} \label{eq:15}
    \begin{aligned} 
 &\frac{\partial F(p,x)}{\partial p} = \frac{\partial E(p,x)}{\partial p} -h\frac{\partial I(p,x)}{\partial p} - b\frac{\partial L(p,x)}{\partial p} = 0,\\
 &b-c+p-(h+b+p)\sum_{k=0}^{x^*-1}\frac{\lambda^k(p)e^{-\lambda(p)}}{k!} \ge  0, \\
 &b-c+p-(h+b+p)\sum_{k=0}^{x^*}\frac{\lambda^k(p)e^{-\lambda(p)}}{k!} < 0,
    \end{aligned}
\end{equation}
where $x^*$ is a positive integer.
\end{theorem}

Previous literature \citep{federgruen1999combined} established the joint concavity of the profit function $F(p,x)$ with respect to price and inventory under Assumptions \ref{asmp:6}(i) and \ref{asmp:7}(i). The Assumption \ref{asmp:6}(i)  represents only a class of continuous distributions. The Assumption \ref{asmp:7}(i) means that the revenue collected in each period depends on the demand volume instead of the sales amount. In this paper, we cannot derive this conclusion primarily because we have revised these conditions to Assumption \ref{asmp:6}(ii) and Assumption \ref{asmp:7}(ii). 
\begin{assumption}\label{asmp:6}
    (i) $d(p,\epsilon) = \gamma(p)\epsilon + \delta(p);$ (ii) $d \sim Pois(\lambda(p))$.
\end{assumption}
\begin{assumption}\label{asmp:7}
    (i) $E(p) = p\mathbb{E}[d]=p\lambda(p);$  (ii) $E(p,x) = p\mathbb{E}[\min(d,x)].$
\end{assumption}
Notably, altering Assumptions \ref{asmp:6}(i) or \ref{asmp:7}(i) to Assumptions \ref{asmp:6}(ii) or \ref{asmp:7}(ii) results in the loss of joint concavity established in \citet{federgruen1999combined}. We illustrate this point with the following proposition.
\begin{proposition}\label{thrm:2}
    If either Assumption \ref{asmp:6}(i) or Assumption \ref{asmp:7}(i) is changed to Assumption \ref{asmp:6}(ii) or Assumption \ref{asmp:7}(ii), respectively, the joint concavity of $F(p,x)$ in \citet{federgruen1999combined} will be lost. 
\end{proposition}

\begin{remark}
    A one-to-one mapping from price $p$ to expected demand $\lambda(p)$ is constructed, thus transforming the optimization variable from price to expected demand \citep{bernstein2016simple}. Since affine transformations do not alter concavity or convexity, the conclusions in this paper remain unchanged due to the Poisson demand.
\end{remark}
Proposition \ref{thrm:2} illustrates the limitations of the traditional model. Even slight modifications to the assumptions can negate joint concavity, rendering the original model's optimal strategies invalid.

Under new assumptions, there are still some conclusions consistent with the finding in previous literature such as the following proposition \citep{federgruen1999combined}.
\begin{proposition} \label{proposition:3}
    Given the inventory $x$ and by solving the first equation in condition (\ref{eq:15}), we can determine $p(x)$ to maximize $F(p,x)$. The function $p(x)$ is a non-increasing function with respect to $x$.
\end{proposition}

\begin{remark}
Figure \ref{fig:non-concavity} in Appendix C clearly shows that $F(p, x)$ is not jointly concave with respect to $p$ and $x$, because  the value of $F(p,x)$ at points between $A$ and $B$ lies below the corresponding linear interpolation of the function values at $A$ and $B$, indicating a violation of joint concavity which is crucial for deriving optimal policies like BSLP. Without it, optimal joint pricing and replenishment decisions cannot be directly obtained from the profit function's first-order conditions. Finding the global maximum requires identifying and comparing all local maxima, a process whose computational complexity increases exponentially with the search space. Consequently, more advanced optimization algorithms become necessary for efficient solutions.
\end{remark}

\subsubsection{Two-timescale Stochastic Approximation}
In commercial operations, the frequencies of product replenishment and pricing decisions often differ. The convergence of a two-timescale algorithm can be guaranteed by the respective concavity of the profit function. This concavity can now be utilized to guide the algorithm design. For the case where the frequency of pricing decisions exceeds that of replenishment decisions, two-timescale algorithms can adjust the relative updating rates of the price and inventory in the algorithm and use a faster search speed to update $p$ in cases where prices are determined with a higher frequency than inventory. Conversely, for some products where replenishment decisions are made more frequently than pricing adjustments, two-timescale algorithms use a faster search speed to update $x$ in these cases.

We illustrate the algorithm when pricing decisions are more frequent than replenishment decisions. The single-period profit function is  
$$F(p, x) =  p\mathbb{E}[\min(d,x)] - (h+b)\mathbb{E}[(x-d)^+] -b\lambda + bx - c(x-x_0),$$
where inventory $x$ changes relatively slower than price $p$ in this case. We conduct simultaneous descents for $p$ and $x$ in two coupled iterations, with the update rate of $p$ being faster than that of $x$. We express the partial derivatives of the profit function in Eq.(\ref{eq:16}) using the form of expectation, allowing us to derive their unbiased estimators, as outlined in the proposition below.

\begin{proposition}\label{pro:4}
Let $g(p,x)$ and $h(p,x)$ denote the partial derivatives of the profit function in Eq.(\ref{eq:16}) with respect to $p$ and $x$, then 
\begin{equation*}
    \begin{aligned} 
g(p,x) &\doteq \frac{\partial F(p,x)}{\partial p} = \mathbb{E}[\min(d,x)] + p\mathbb{E}\bigg[\frac{\partial \log\mathbb{P}(d)}{\partial p} \min(d,x)\bigg] -(h+b)\mathbb{E}\bigg[\frac{\partial \log \mathbb{P}(d)}{\partial p} (x-d)^+\bigg] -b\frac{\partial \lambda}{\partial p }, \\
h(p,x) &\doteq \frac{\partial F(p,x)}{\partial x} =  b-c+p - (h+b+p)\mathbb{E}[1_{d \le x}],
    \end{aligned}
\end{equation*}
where $\mathbb{P}(d)$ is the probability density function of random variable $d$. 
\end{proposition}

Due to Proposition \ref{pro:4}, unbiased stochastic gradient estimators for $g(p,x)$ and $h(p,x)$ can be obtained using the Monte Carlo method and are denoted as $\hat{g}(p,x)$ and $\hat{h}(p,x)$, respectively. The iteration process of the two-timescale algorithm can be formulated as follows:
\begin{equation}\label{eq:19}
    \begin{aligned} 
p_{k+1} &= p_k + \alpha_k\hat{g}(p_k,x_k) \\ &  \doteq p_k + \alpha_k\bigg(\min(d_{k},x_k)+p_k\frac{\partial \log \mathbb{P}(d_{k})}{\partial p}\min(d_{k},x_k) -(h+b)\frac{\partial \log \mathbb{P}(d_{k})}{\partial p}(x_k-d_{k})^+-b\frac{\partial \lambda}{\partial p }\bigg),\\
x_{k+1}&= x_k+\beta_k\hat{h}(p_k,x_k) \doteq x_k+\beta_k\bigg((b-c+p_k)-(h+b+p_k) 1_{d_{k}\le x_k}\bigg),
    \end{aligned}
\end{equation}
where $d_{k}$ is sampled from a Poisson distribution with parameter $\lambda(p_k)$. Note that $\mathbb{E}[\hat{g}(p,x)] = g(p,x)$ and $\mathbb{E}[\hat{h}(p,x)] = h(p,x)$. We could reduce the variance by taking more samples each time, while a single sample is sufficient for convergence to the optimum. Step size $\alpha_k$ represents the time scale for the first iteration, and $\beta_k$ represents the time scale for the second iteration. The conditions $\sum_{k=1}^{\infty}\alpha_k = \sum_{k=1}^{\infty}\beta_k =\infty$, $\sum_{k=1}^{\infty}\beta_k^2 < \sum_{k=1}^{\infty}\alpha_k^2 < \infty$, and $\frac{\beta_k}{\alpha_k} \rightarrow 0$ as $k$ tends to infinity are satisfied, indicating that the first time scale progresses faster than the second time scale.

The convergence of the two-timescale algorithm is guaranteed, analogous to the theoretical results in \cite{borkar2009stochastic} and \cite{Kushner1997StochasticAA}. Specifically, we present the following lemma to analyze the convergence of $p_k$ in the faster time scale, leveraging the concavity of the revenue function with respect to $p$.
\begin{lemma}\label{lemma3}
The iterative sequence $\{(p_k,x_k)\}$ generated by the iteration Eq.(\ref{eq:19}) converges to a set which is defined as below:
\begin{equation*}
    (p_k,x_k) \rightarrow \{(p^{\star}(x),x)\},
\end{equation*}
where $p^{\star}(x)$ is a well defined function showing that when $x$ is fixed, the value of $p$ to make $g(p,x)$ equal to zero. In other words, this lemma means $|p_k - p^{\star}(x_k)|\rightarrow 0$ almost surely, that is, $\{p_k\}$ asymptotically tracks $p^{\star}(x_k)$. 
\end{lemma}

Then, we will combine the analysis of the second time scale and show that the sequence $\{(p_k,x_k)\}$ converges to a fixed point. Considering the second scale, the variable $x$, previously considered fixed, should now be treated as the time-varying $x(t)$. Based on Lemma \ref{lemma3}, for sufficiently large $t$, $p(t)$ closely tracks $p^{\star}(x(t))$. We then substitute it into the second scale and focus on the following ODE:
\begin{equation} \label{eq:22}
\dot{x}(t) = h(p^{\star}(x(t)),x(t)).
\end{equation}

\begin{theorem}\label{theorem3}
    The iterative sequence $\{(p_k,x_k)\}$ generated by the iteration Eq.(\ref{eq:19}) converges to some limit point of the ODE (\ref{eq:22}), i.e.,
    \begin{equation*}
        (p_k,x_k) \rightarrow (p^{\star}(x^{\star}),x^{\star}).
    \end{equation*}
\end{theorem}
\begin{proof}{Proof}
Since $h(p,x) = b-c+p - \sum_{k=0}^{\lfloor x \rfloor}\frac{\lambda(p)^ke^{-\lambda(p)}}{k!}(h+b+p),$
note that $h(p,x)$ is a piecewise constant function so the ODE (\ref{eq:22}) is equipped with discontinuous right hand side, for which standard solution concepts are introduced in \cite{Filippov1988DifferentialEW}. It is obvious that $h(p,x)$ satisfies a linear growth condition: $|h(p,x)|\le K(1+|x|)$ for some $K>0$. Define $H(p,x)\doteq \cap_{\epsilon>0} \bar{co}(\{h(p,y):|y-x|<\epsilon \}),$ where $\bar{co}$ is the notation of convex closure. Then the ODE limit gets replaced by a differential inclusion limit: $\dot{x}(t) \in H(p^{\star}(x(t)),x(t)).$
     By Corollary 4 in \citet[Chap.5]{borkar2009stochastic}, the sequence $\{x_k\}$ generated by the iteration Eq.(\ref{eq:19}) converges to (possibly path dependent) the limit point of Eq.(\ref{eq:22}) almost surely. Furthermore, the iterative sequences $\{p_k\}$ and $\{x_k\}$ generated by iterations Eq.(\ref{eq:19}) converge to some limit point of the ODE below by Theorem 2 in \citet[Chap.6]{borkar2009stochastic}. 
 \begin{equation*} \small
\dot{p}(t) = g(p(t),x(t)),\quad \dot{x}(t)=h(p(t),x(t)). \Halmos
\end{equation*} 
\end{proof}
The limit point of ODE in Lemma \ref{lemma3} and Theorem \ref{theorem3} satisfy the condition $g(p,x) = 0$ and $h(p,x) =0$, which shows $\frac{\partial F(p,x)}{\partial p} = 
 \frac{\partial F(p,x)}{\partial x} = 0$ by definition of $g(p,x)$ and $h(p,x)$. Therefore, the limit point of ODE is a local optimum in the optimization problem. Figure \ref{fig:two-scale-simple} in Appendix C shows that the two-timescale algorithm approximately converges to $p=55$ and $x=5$ under certain parameters, which fully satisfy the conditions in Theorem \ref{theorem:1}.

 Furthermore, we study the convergence rate of the two-timescale algorithm.
\begin{theorem}
    Assume that the function $p^{\star}(x)$ is a Lipschitz function, and then we have the order of the mean squared error for the sequence $\{x_k\}$ generated by the iteration Eq.(\ref{eq:19}):
    \begin{equation*}
        \mathbb{E}[|p_k-p^{\star}(x_k)|] = O(\frac{\beta_k}{\alpha_k}) 
        + O(\alpha_k^{\frac{1}{2}}).
    \end{equation*}
\end{theorem}

\subsection{Multi-period Problem Property}
When it comes to multi-period decision problems, the problem is formulated as MDP through Bellman equation (\ref{eq:10}), with the goal of determining the optimal policy for the current state in each period. \par 
At each step excluding the final period $T$, the optimization of the Bellman equation (\ref{eq:10}) includes a future discounted expectation term with respect to the next state. To account for this, we need to find the optimal policy $a_{t+1}(s_{t+1})$ under every state $s_{t+1}$ in the $t+1$ period, calculate the $V^{\pi_\star}(s_{t+1})$, and then take the expectation $\mathbb{E}_{s_{t+1}}[V^{\pi_\star}(s_{t+1})]$. This transforms the problem into a functional optimization task, as we need to find a function $a_{t+1}(s_{t+1})$ that maximizes $\mathbb{E}_{s_{t+1}}[V^{\pi_\star}(s_{t+1})]$. The computational complexity involved in solving Eq.(\ref{eq:10}) by backward induction is detailed in Proposition \ref{pro:1}.  
\begin{proposition}\label{pro:1}
Suppose the price of each product is discretized into P values and the replenishment quantity is discretized into $Q$ values; there are $K$ regressors in Eq.(\ref{eq:1}), the maximum dimension and the minimum dimension of the regressors are $R_{max}$ and $R_{min}$. The lower  and upper bounds of the total computational complexity needed to recursively solve the optimality equation (\ref{eq:10}) via DP are $O(P^T Q^T (R_{min})^{KT})$ and $O(P^T Q^T (R_{max})^{KT})$, respectively.
\end{proposition}
It can be observed that the  traditional dynamic programming technique suffers from the curse-of-dimensionality, rendering it computationally intractable for sizable problems. The curse-of-dimensionality arises from the high dimensions of the state space which grows in the number of products, the action space which includes all possible pricing and replenishment combinations and the transition space dependent on the number of possible demand realizations.
To address such complex scenarios, particularly in high-dimensional, continuous spaces typical of functional optimization tasks, DRL can be applied. By interacting with the environment, DRL learns to make decisions based on observed feedback while leveraging the power of deep neural networks.

\subsection{Two-timescale DRL Approach}
We use two agents to make pricing and replenishment decisions,  respectively, and integrate the two-timescale algorithm with HAPPO. To achieve optimal pricing and replenishment  decisions, each agent should consider the influence of the action on not only its own reward but also the reward of other agents. Thus, the MDP constructed in Section 3.1 should be transformed into an MG.

\subsubsection{Markov Game Formulation}
We formulate the dynamic pricing and replenishment problem with $T$ periods as an MG, which can be defined as a tuple $<N,\mathcal{S},\bold{\mathcal{A}},\mathcal{P},\mathcal{R},\gamma>$. Here $N \in \mathbb{N}^+ $ is the number of agents; $\mathcal{S}$ is the finite state space; $\bold{\mathcal{A}}=\prod_{i=1}^n \mathcal{A}^i$ is the product
of finite action spaces of all agents, known as the joint action space, where $\mathcal{A}^i$ is the set of actions for agent $i$; $\mathcal{P}: \mathcal{S} \times \bold{\mathcal{A}} \times \mathcal{S} \rightarrow [0,1]$ is the transition probability function that maps the current global state, global action (including all agents’ actions) and next possible state to a transition probability; $\mathcal{R}: \mathcal{S} \times \bold{\mathcal{A}} \rightarrow \mathbb{R}$ is the reward function that maps current global state and global action to a payoff; and $\gamma \in (0,1]$ is the discount factor. The agents interact with the environment according to the following protocol: in each period $t$, the agents are at state $s_t \in \mathcal{S}$; every agent $i$ takes an action $a^i_t \in \mathcal{A}^i$, drawn from its policy $\pi^i(\cdot |s_t)$, which together with other agents’ actions gives a joint action $\bold{a}_t \in \bold{\mathcal{A}}$, drawn from the joint policy $\bold{\pi}(\cdot |s_t)=\prod_{i=1}^n \pi^i(\cdot |s_t)$. After taking actions $\bold{a}_t$, the agents receive a joint reward $r_t = \mathcal{R}(s_t|\bold{a}_t)$, which is set as the sales profit in period $t$; see Eq.(\ref{eq:6}).
Afterwards, the system transfers to the next state $s_{t+1}$ with probability $\mathcal{P}(s_{t+1}|s_t,\bold{a}_t)$, where the state transition function $\mathcal{P}$ is determined by the dynamics of the pricing and replenishment model, as detailed in Section 3.2. A trajectory is generated by repeating the interaction until period $T$ is reached. All agents update their policies cooperatively, aiming to maximize the expected total reward:
\begin{equation}
    J(\pi) = \mathbb{E}_{a_{1:T}^1 \sim \pi^1,...,a_{1:T}^n \sim \pi^n,s_{1:T} \sim \rho^{1:T}_\pi} \Bigg[ \sum_{t=1}^T \gamma^t r_t\Bigg],
\end{equation}
where $a_{1:T}^i \sim \pi^i$ represents that each action in action sequence $(a_1^i,...,a_T^i)$ of actor $i$ follows $\pi^i$, and $\rho^t_{\pi}$ is the marginal state distribution at time $t$ with the joint policy $\pi$.

\subsubsection{Solution}
Now we show how to use our fast-slow dual-agent DRL algorithm (FSDA) to construct pricing and replenishment policies.
FSDA applies an on-policy PPO for each agent and follows a CTDE structure. We use the subscripts $\theta$ and $\phi$ to represent the parameters of policy and value functions, respectively. Each agent $i$ has a policy function $\pi_{\theta}^i$. The value function $v_{\phi}$ maps all agents’ state to state value $v_{\phi}(s_t)$. Both the policy function and the value function are constructed using neural networks, known as the actor network and the critic network, respectively. To capture useful historical information such as demand fluctuations and price changes during the lead time, we choose a Recurrent Neural Network (RNN) with two stacked Gate Recurrent Units (GRUs) to construct our actor and critic networks \citep{liu2022multi}. Thus, the agent can fully take historical information into account and make better decisions.\par

\begin{figure}[h!]
  \centering
  \caption{Structure of actor network}
  \label{fig:Network}
  \includegraphics[width=0.7\textwidth]{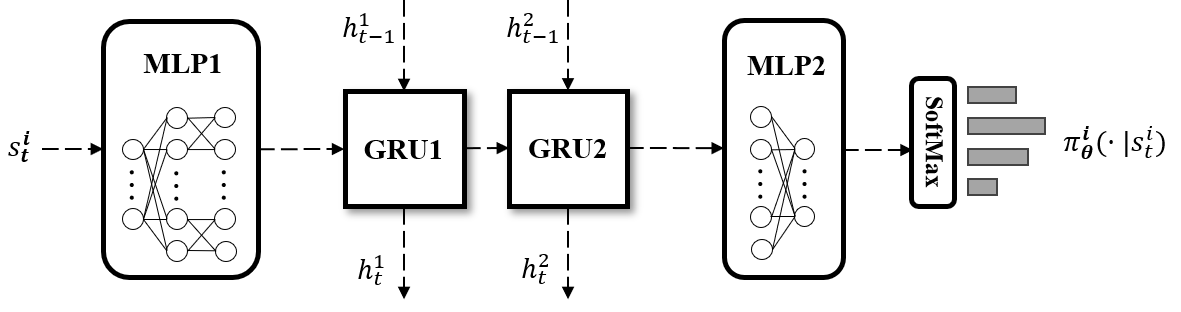} 
  
\end{figure}

We illustrate the diagram for the actor network in Figure \ref{fig:Network}. For the actor network, we use a multi-layer perceptron (MLP) with two hidden layers, namely MLP1, to encode the actor $i$'s state $s_t^i$ into high-dimensional features. The first GRU, namely GRU1, receives these features and combines a hidden state $h_{t-1}^1$ as the input and outputs a hidden state $h_t^1$ and some high-dimensional features which are served as the input of the second GRU (GRU2). GRU2 outputs an updated hidden state $h_t^2$ and new high-dimensional features. These features are then fed into the second multi-layer perceptron (MLP2) to match the size of the action space and mapped into the probability distribution $\pi_{\theta}^i(\cdot|s_t^i)$ of actions by a softmax layer. For the critic network, its only difference from the actor network lies in that features produced by GRU2 are directly mapped into the state value.\par
For the actual training process, we use $\theta_m$ and $\phi_m$ to represent parameters for actor and critic
networks after $m$ updates. For each episode, we simulate the MG corresponding to the dynamic pricing and replenishment problem for $T$ periods with actors’ current policies ($\pi_{\theta_m}^1$,...,$\pi_{\theta_m}^n$) and collect all actors’
states, actions and rewards as the training data.  Then, if the current episode is divisible by the ratio $k$ of the two time scales, each actor’s parameters denoted by $\theta_m$ are updated to $\theta_{m+1}$ using the collected data, respectively. Otherwise, only the first actor’s parameters $\theta_m$ are updated to $\theta_{m+1}$. Note that the first actor can be either a pricing or a replenishment agent, depending on which decision needs to be updated more quickly. After that, the critic's parameters $\phi_m$ are updated to $\phi_{m+1}$. The ratio $k$ of the two time scales tends to infinity as the algorithm iterates. That is to say, $k$ is a function of $m$, for example, $k=m/2$. A pseudocode for the detailed simulation and parameter update process is presented
in Appendix E. \par

In the application of the algorithm, we also incorporate several performance enhancement techniques to optimize its effectiveness, such as batch advantage normalization, reward scaling, and orthogonal initialization.\par

\subsection{Theoretical Analysis}
For the FSDA in this section, we show its convergence below. Notice that $(\theta_m^2,\phi_m)$ is updated at every step, while $\theta_m^1$ is updated once every $k$ steps. When using the Adam optimizer, with a step size that satisfies the descent condition, we can consider $(\theta_m^2,\phi_m)$ to be on a faster time scale and $\theta_m^1$ on a slower time scale. The iterative formulas extracted from the algorithm are as follows: 

\begin{equation}\small \label{eq:23}
    \begin{aligned} 
    \theta_{m+1}^1 &= \theta_m^1 + \beta_m \frac{\partial}{\partial \theta_m^1}  (L_p^1+L_e^1) = \theta_m^1 + \beta_m( h_1(\theta_m^1,\phi_m) + N_m^{(1)}) ,\\
\theta_{m+1}^2 &= \theta_m^2 + \alpha_m \frac{\partial}{\partial \theta_m^2}  (L_p^2+L_e^2) = \theta_m^2 + \alpha_m (h_2(\theta_m^1,\theta_m^2,\phi_m)+N_m^{(2)})  ,\\
\phi_{m+1}& = \phi_m +\alpha_m \frac{\partial}{\partial \phi_m} L_c = \phi_m + \alpha_m(h_3(\theta_m^1,\theta_m^2,\phi_m)+N_m^{(3)}),\\
    \end{aligned}
\end{equation}
where $k_m=\frac{\alpha_m}{\beta_m} \rightarrow \infty$ due to our setting and the deterministic functions $h_1(\theta^1,\phi)$, $h_2(\theta^1,\theta^2,\phi)$ and $h_3(\theta^1,\theta^2,\phi)$ are as follows:
\begin{equation} \small \label{eq:23-1}
    \begin{aligned} 
h_1(\theta^1,\phi) &= \mathbb{E}_{s \sim \rho_{\pi_{\theta_m}}, a \sim {\pi_{\theta_m}}} \Bigg[ \min \Big( \frac{\pi_{\theta_{m+1}}^{1}(a_t^{1}|s_t^{1})}{\pi_{\theta_{m}}^{1}(a_t^{1}|s_t^{1})} A_t^{1}, clip\Big( \frac{\pi_{\theta_m}^{1}(a_t^{1}|s_t^{1})}{\pi_{\theta_{m-1}}^{1}(a_t^{1}|s_t^{1})} , 1-\epsilon, 1+\epsilon \Big)  A_t^{1} \Big) \Bigg],\\
 h_2(\theta^1,\theta^2,\phi) &= \mathbb{E}_{s \sim \rho_{\pi_{\theta_m}}, a \sim {\pi_{\theta_m}}} \Bigg[ \min \Big( \frac{\pi_{\theta_{m+1}}^{2}(a_t^{2}|s_t^{2})}{\pi_{\theta_{m}}^{2}(a_t^{2}|s_t^{2})} F_t^{2} A_t^{2}, clip\Big( \frac{\pi_{\theta_m}^{2}(a_t^{2}|s_t^{2})}{\pi_{\theta_{m-1}}^{2}(a_t^{2}|s_t^{2})} , 1-\epsilon, 1+\epsilon \Big) F_t^{2} A_t^{2} \Big) \Bigg],\\
h_3(\theta^1,\theta^2,\phi) & = \mathbb{E}_{s \sim \rho_{\pi_{\theta_m}}} [ (v_{\phi}(s_t)-r_t)^2 ].\\
    \end{aligned}
\end{equation}
where $clip(x, \text{min}, \text{max})$ confines $x$ to the interval $[\text{min}, \text{max}]$, returning $\text{min}$, $\text{max}$, or $x$ accordingly as $x$ falls below $\text{min}$, exceeds $\text{max}$, or lies between them.
\begin{lemma}  \label{lemma4}
Suppose that the parameter space for $(\theta_m^1,\theta_m^2,\phi_m)$ is a convex bounded set. Also suppose that $h_1$, $h_2$ and $h_3$ are Lipschitz functions and $N_m^{(1)}$, $N_m^{(2)}$ and $N_m^{(3)}$ are martingales with finite conditional second order moment. The step sizes $\alpha_m$ and $\beta_m$ satisfy the conditions $\sum_{k=1}^{\infty}\alpha_k = \sum_{k=1}^{\infty}\beta_k =\infty$, $\sum_{k=1}^{\infty}\beta_k^2 < \sum_{k=1}^{\infty}\alpha_k^2 < \infty$, and $\frac{\beta_k}{\alpha_k} \rightarrow 0$ as $k$ tends to infinity. The iterative sequence $\{(\theta_m^1,\theta_m^2,\phi_m)\}$ generated by the iteration Eq.(\ref{eq:23}) converges to a set which is defined as below:
\begin{equation*} \small 
    (\theta_m^1,\theta_m^2,\phi_m) \rightarrow \{(\theta_m^1,\theta_m^{2\star}(\theta_m^1),\phi_m^{\star}(\theta_m^1))\},
\end{equation*}
where $\theta_m^{2\star}(\theta_m^1)$ and $\phi_m^{\star}(\theta_m^1)$ are well-defined functions, similar to $p^*(x)$ defined in Lemma \ref{lemma3}. In other words, this lemma means $|\theta_m^2 - \theta_m^{2\star}(\theta_m^1)|\rightarrow 0$ and $|\phi_m - \phi_m^{\star}(\theta_m^1)|\rightarrow 0$ almost surely, that is, $\{\theta_m^2\}$ and $\{\phi_m\}$ asymptotically track $\theta_m^{2\star}(\theta_m^1)$ and $\phi_m^{\star}(\theta_m^1)$, respectively. 
\end{lemma}

\begin{theorem}\label{theorem5}
    In the condition of Lemma \ref{lemma4}, the iterative sequence $\{(\theta_m^1,\theta_m^2,\phi_m)\}$ generated by the iteration Eq.(\ref{eq:23}) converges to some limit point of the ODE $\dot{\theta}^1(t) = h_1(\theta^1(t),\phi^{\star}(\theta^{1}(t))),$ i.e., $(\theta_m^1,\theta_m^2,\phi_m) \rightarrow (\theta^{1\star}, \theta^{2\star}(\theta^{1\star}),\phi^{\star}(\theta^{1\star}))$.
\end{theorem}
The proofs of Lemma \ref{lemma4} and Theorem \ref{theorem5} are similar to Lemma \ref{lemma3} and Theorem \ref{theorem3} and we omit them.
\subsection{Algorithmic Complexity Analysis}
In this section, we analyze the computation complexity of the FSDA algorithm shown in Algorithm \ref{Algo:1}. The computational complexity is associated with the scale of the pricing and replenishment problem under consideration, including the length of the decision period and the dimensions of the state and action. Additionally, the computational complexity is also influenced by the specific design of the neural networks in the FSDA, such as the input/output width and the number of hidden layers of the MLP, the number of GRUs and the number of episodes for training. We provide the computation complexity of the model training and the model inference in Proposition \ref{pro:2}.

\begin{proposition}\label{pro:2}
Suppose that the dimension of state $s_t^i$ for each agent $i \in \{1, . . .,N\}$ in each period $t \in \{1, . . .,T\}$ is $D_I$, the dimension of the action $a_t^i$ is $D_A$, the width of MLP1's hidden layers are both $W_1$, and the input width of MLP2 is $W_2$. Then the computational complexity of model training with $M$ episodes is
\begin{equation}
    O\Big(\big((D_I+W_1) W_1+(W_1+3W_2)W_2+ W_2 D_A \big)NMT \Big).
\end{equation}
When M equals 1, it represents the computational complexity of model inference.
\end{proposition}

\subsection{Baseline Policies}
We implement three classic heuristic policies as baselines to evaluate the performance of our algorithm. The first is the BSLP policy, which operates on two principles: if the beginning inventory $x_t$ at time $t$ falls below the base stock level $y^*_t$, the inventory is increased to that level while charging a specific list price $p^*_t$; second, if $x_t$ exceeds the base stock level $y^*_t$, no replenishment occurs, and the product is sold at price $p^*_t(x_t)$, which decreases as the inventory level $x$ increases. 
\par
The second baseline policy, known as the (s,S,p) policy, follows these principles: whenever the on-hand inventory $x_t$ at the start of period $t$ goes below $s$, the seller replenishes up to level $S$; if $x_t > s$, the seller orders nothing and charges price $p(x_t)$ that decreases with inventory level $x_t$.  \par
The third baseline policy, proposed by \citet{bernstein2016simple}, is referred to as ‘Myopic’ in this paper. This heuristic involves a myopic pricing policy that generates each period’s price $p_t$ as a function of the beginning inventory level $x_t$ at time $t$ and a base-stock policy for inventory replenishment. The inventory position weights the on-hand and pipeline inventory according to a factor that reflects the sensitivity of price to the net inventory level.  \par
\citet{tunc2011cost} propose two approaches  to extend these policies for use under non-stationary demand. The first approach is to find a stationary demand pattern which best fits the original demand for the given problem, which requires a search through various demand functions. The second approach is to find a stationary policy that provides the minimum cost for the actual non-stationary demand, which requires a search on
various policy parameters. We opt for the first method and simulate 10,000 demand-price pairs to identify the most suitable demand pattern. Subsequently, the pricing and replenishment decisions are jointly optimized using these heuristic methods. \par

\subsection{Experimental Results}
Our demand model described in Eq.(\ref{eq:2}) originates from the work of \citet{schlosser2018dynamic}, who analyze stochastic dynamic pricing models in competitive markets with multiple offer dimensions such as product qualities and seller ratings. In the subsequent study, detailed in \citet{schlosser2019dynamic}, they incorporate market reference prices, which reflect historical market prices, as well as competitor strategies and price rankings. This enhancement significantly extends the previous demand model, capturing more effectively the dynamics of competition on demand. The authors also calibrate this model using a real dataset from Amazon. Therefore, we utilize their refined model to simulate demand generation in this section. \par 
\citet{schlosser2019dynamic} consider the following regressors related to sales:\par
 \begin{itemize}[align=left]
    \item[(i)] constant / intercept: $\kappa_1(p,o,j) = 1 $,
    \item[(ii)] rank of price $p$ compared to price $o$: $\kappa_2(p,o,j) = 1+(1_{\{o < p\}}+1_{\{o = p\}}) / 2$, 
    \item[(iii)] price gap between price $p$ and price $o$: $\kappa_3(p,o,j) = o-p$,
    \item[(iv)] total number of competitors: $\kappa_4(p,o,j) = 1$,
    \item[(v)] average price level: $\kappa_5(p,o,j) = (p+o)/2$,
    \item[(vi)] reference price effect: $\kappa_6(p,o,j) = p-j$,
\end{itemize}
where $1_{\{o < p\}}$ is an indicator function that equals 1 when $o < p$, and 0 otherwise. For specific values of $\Delta$ and the update rules for the reference price $j$, we refer to \citet{schlosser2019dynamic} who point  out that overall customer behavior reflects that lower price levels increase sales and price ranking is the most noticeable feature. In our experimental setting, the competitor will slightly lower its price each time based on our pricing until it reaches the lowest price in the price range, after which it will raise it to the highest price. We also test different competitor pricing strategies, such as randomly selecting prices from a uniform distribution within a specified range, and find that the competitor's pricing strategy has a negligible effect on our experimental results.\par

Considering that the (s,S,p) policy is designed for pricing and replenishment problem with fixed ordering cost, we conduct experiments on pricing and replenishment with fixed cost and modeling it into product’s profit as
\begin{equation}
    r_t = p_t S_t - h I_t - b L_t - c q_t - f 1_{\{q_t>0\}},
\end{equation}
where $f$ is the unit fixed cost and $1_{\{x > 0\}}$ is an indicator function which takes 1 when $x > 0$ and takes 0 otherwise.
\par 
Based on the demand Eq.(\ref{eq:2}), we design the following experiments to evaluate the performance of our algorithm when dealing with the dynamic pricing and replenishment under competition problem: (a) demand is backlogged when inventory is insufficient, without considering fixed ordering costs; (b) demand is lost when inventory is insufficient, without considering fixed ordering costs; (c) demand is backlogged when inventory is insufficient, with consideration of fixed ordering costs; (d) demand is lost when inventory is insufficient, with consideration of fixed ordering costs. 
In all experiments, the lead time was set to 3 periods. The experimental results show that our algorithm performs well across various scenarios, regardless of whether unmet demand is lost or whether fixed ordering costs are considered. \par
 \begin{figure}
  \centering
  \caption{Comparison of reward under different scenarios}
  \includegraphics[width=0.7\textwidth]{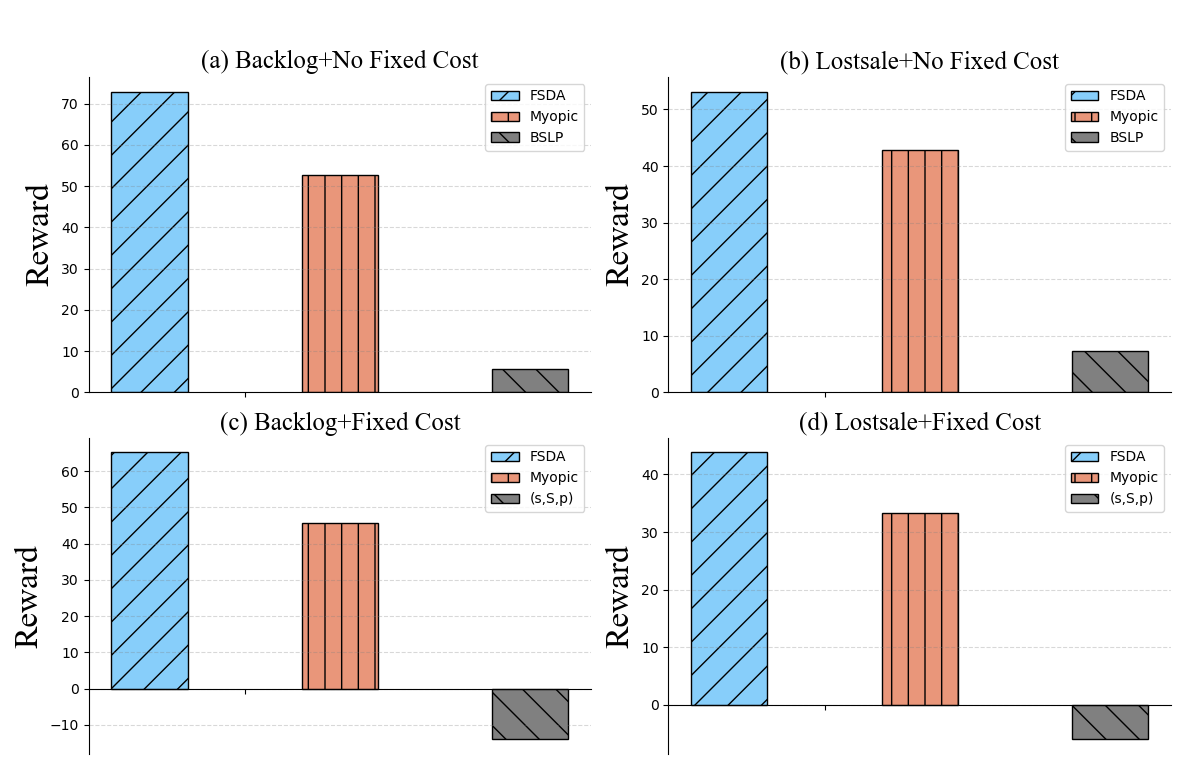} 
  
  \label{fig:Comparison}
\end{figure}

Comparison of reward in Figure \ref{fig:Comparison} shows that FSDA achieves the highest sales profits (rewards) in all cases. The numerical results are analyzed from two aspects:
\subsubsection{Unmet Demand: Backlogged or Lost-sales} \par
    The FSDA consistently outperforms heuristic algorithms in scenarios with backlogged demand and lost demand. Specifically, in the case of unmet demand being backlogged, FSDA achieves an average profit improvement of 40.37\% compared to the Myopic policy and 1167.92\% compared to BSLP. In the case of unmet demand being lost, FSDA achieves an average profit improvement of 27.78\% compared to the Myopic policy and 631.18\% compared to BSLP. Since the (s,S,p) policy shows negative profits, it is not included in the comparison here. Among heuristic methods, the Myopic policy performs better than others; this may be due to its consideration of the impact of purchase costs in the next period on the current period’s pricing decision \citep{bernstein2016simple}. Because the FSDA is capable of considering longer-term future returns, it more effectively manages situations where unmet demand can be fulfilled in subsequent periods. Meanwhile, the BSLP and (s, S, p) policies are clearly at a disadvantage in comparison. The two policies react only when inventory falls below a threshold, potentially resulting in missed opportunities during peak demand periods or overstock during low demand periods, unlike the FSDA, which continuously learns and predicts demand patterns. In lost-sales scenarios, where unmet demand leads to a permanent loss, the higher rewards obtained by FSDA highlight its proficiency in mitigating the negative impacts of lost-sales, emphasizing its ability to make more informed and adaptive decisions in such challenging conditions.  
  
\subsubsection{Cost Consideration: No Fixed Cost or Fixed Cost}
        The FSDA provides higher sales profits than the heuristic methods in scenarios with and without fixed costs. Specifically, without considering the fixed ordering cost, FSDA achieves an average profit improvement of 30.90\% compared to the Myopic policy and an average improvement of 899.55\% compared to BSLP. When considering the fixed ordering cost, FSDA achieves an average profit improvement of 37.25\% compared to the Myopic policy, while the (s,S,p) policy yields negative profits and is therefore not compared. BSLP maintains a base stock level and adjusts prices based on inventory, which does not effectively handle unexpected demand surges or drops. The (s,S,p) policy, specifically designed to handle scenarios with fixed ordering costs, does not show the expected advantage when such costs are introduced. This could be attributed to the fact that the (s,S,p) policy being structured to reduce ordering frequency may fail to capture the benefits of more dynamic pricing and inventory adjustments that could potentially capitalize on market fluctuations. The Myopic policy, on the other hand, adjusts prices based on current inventory and purchase costs, showing slightly better flexibility in reacting to market changes. Although fixed costs significantly impair all strategies' overall profitability, the impact on the Myopic and FSDA is less pronounced.
        
\begin{figure}[h!]
\centering

\caption{Pricing and beginning inventory under different policies}
\begin{subfigure}[b]{0.5\textwidth}
   \caption{\small Myopic policy}
   \includegraphics[width=\textwidth]{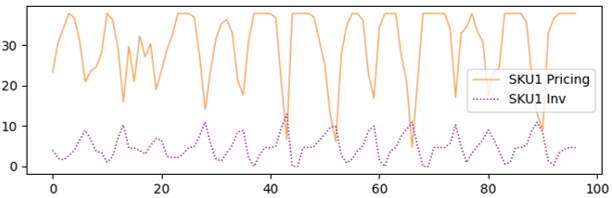}
   
   \label{fig:sub1}
\end{subfigure}

\begin{subfigure}[b]{0.5\textwidth}
   \caption{\small Our policy}
   \includegraphics[width=\textwidth]{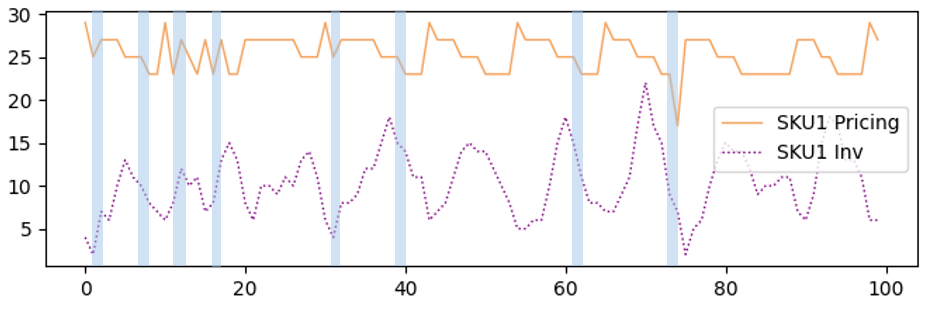}
   
   \label{fig:sub2}
\end{subfigure}

\label{fig:pricing and beginning inv}
\end{figure}

 Myopic generally performs better than the other heuristics, while it falls short of FSDA's performance. 
We illustrate the pricing and the beginning inventory at each period over 100 periods for the same product under different policies (Myopic and FSDA) in Figure \ref{fig:pricing and beginning inv}. Under the Myopic policy, pricing and beginning inventory levels at each period consistently follow a predefined relationship. Specifically, pricing is a non-increasing function with respect to the beginning inventory at each period. In contrast, under FSDA, the relationship between pricing and inventory is more complex and flexible. The bars in Figure \ref{fig:sub2} illustrate that pricing is not necessarily a non-increasing function of the beginning inventory at each period. When making pricing decisions in a real market environment, it is essential to consider not only inventory but also various other potential market factors.

\section{Data-driven Demand Modeling}
In this section, we start by training a decision tree-based ML model using real market datasets to forecast market demand. The ML model learns patterns and relationships directly from the data, enabling more accurate demand forecasts and supporting more scientific decision-making. 
 Then, we apply FSDA to optimize the pricing and replenishment strategies based on the trained ML demand model. 
\subsection{Problem Introduction}
\subsubsection{Data and Preprocessing}
The data is sourced from the 2018 CCF Big Data \& Computational Intelligence Contest (CCF BDCI), which provides demand data for the Saudi Arabian market from March 1, 2017, to March 16, 2018, as well as product attribute data (price, brand, and category information), history demand data, and context features (such as season and holidays). \par
We choose LightGBM (Light Gradient Boosting Machine) as our demand forecasting model because of its outstanding performance and efficiency in handling complex data patterns \citep{ke2017lightgbm}. The detailed training process can be found in Appendix D. In order to better capture the trend of demand changes, we extract statistical features related to historical demand. These include aggregate features of daily/monthly product sales, such as maximum, average, and minimum daily/monthly sales, as well as the variance in daily/monthly sales. Considering that temporal factors can influence demand \citep{ehrenthal2014demand}, we incorporate time-related attributes, including seasons, weekends, weekdays, holidays, and their interactions on product demand by creating cross-statistics for different contextual features. Since customers consider product attributes and prices when purchasing products, we also extract features on product attributes and prices. The former includes: the product's first, second, and third-level categories (with lower levels indicating broader scopes), and product brands. The latter includes statistical features of the price, such as the current price, historical average price, and highest price. Additionally, to more accurately predict user purchasing behavior (implying demand occurs), we also collect features related to user behavior, including the number of times a product is added to a shopping cart, the number of times it is favorited or clicked. The feature importance ranking is shown in Figure \ref{fig:feature} in Appendix D, where we find that the ‘Date of Sale’ and ‘Avg. of daily sales of SKU ID stand out as highly influential features. Additionally, the features ‘The ratio of current price to the average price’ and ‘User Behavior’ are also emphasized as having a significant impact, highlighting the importance of analyzing relative prices and user behavior in demand forecasting. Although ‘Product Attributes’ are slightly more influential, they still rank lower on the importance scale, indicating that their impact on market demand predictions is less critical than temporal and sales-related features.  \par

Subsequently, we integrate the trained LightGBM model into the RL training environment. That is, DRL agents make pricing decisions for products, the LightGBM model receives these prices and gives predicted market demand in the current time period. Simultaneously, the current profits are calculated by checking the inventory, following the similar dynamic process referenced in Equations (\ref{eq:3})-(\ref{eq:6}). 

\subsection{Baseline Policy and Experimental Results}

\begin{figure}[h!]
  \centering
  \caption{Comparison of reward under data-driven system}
  \includegraphics[width=0.8\textwidth]{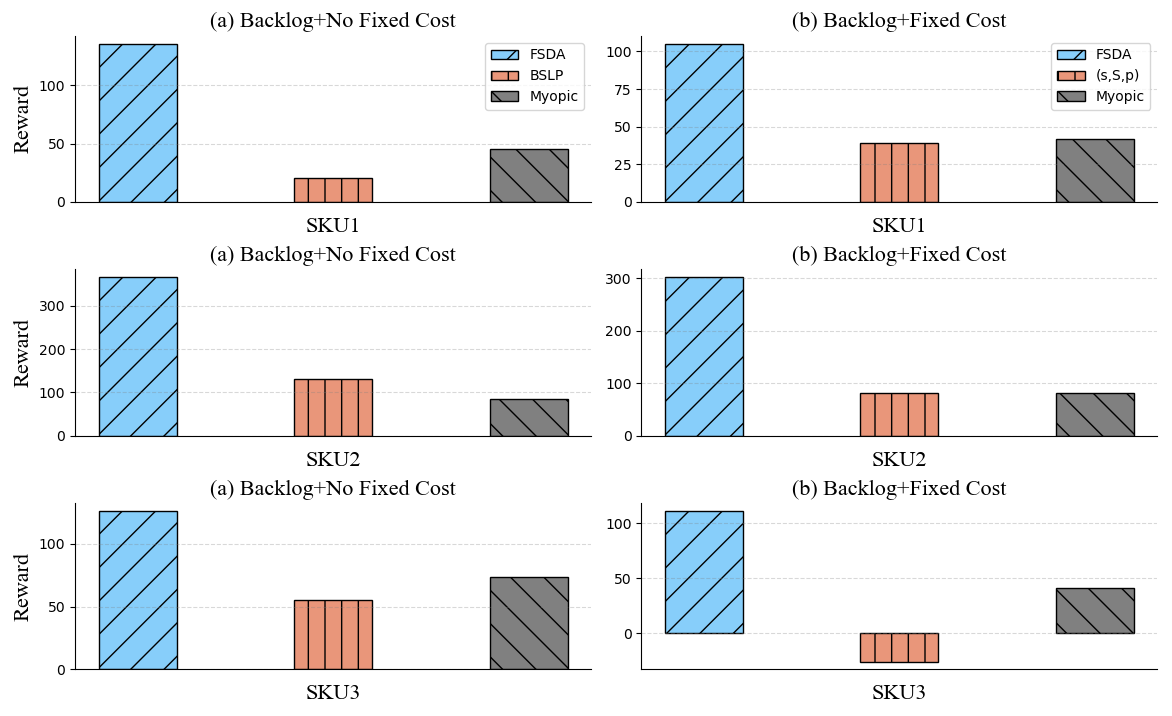} 
  
  \label{fig:HAPPO+Myopic}
\end{figure}
We assume that a ‘period’ in our model corresponds to one day. We select three products from the dataset to optimize their pricing and replenishment decisions, respectively, and compare the performance of the FSDA with three heuristic methods (BSLP, (s,S,p), and Myopic) in the context of data-driven problems. The experimental comparison results, shown in Figure \ref{fig:HAPPO+Myopic}, reveal that FSDA consistently outperforms the other policies in all scenarios, maintaining high profits regardless of the presence of fixed costs. Specifically, without considering the fixed ordering cost, the average sales profit achieved by FSDA increases by 293.83\% compared to the BSLP method and by 197.28\% compared to the Myopic method. In the setting with a fixed ordering cost, FSDA increases average sales profit by 218.47\% compared to the (s,S,p) policy (excluding SKU3, which has negative sales profit under the (s,S,p) policy) and by 199.50\% compared to the Myopic method. Unlike heuristic methods that rely on predetermined assumptions, the FSDA  dynamically adapts to real-time information and demonstrates its superior ability in addressing data-driven challenges. The (s,S,p) policy results in a negative reward for SKU3, indicating it is ill-suited for dynamic market conditions that require rapid adjustments in inventory and pricing. \par

\subsection{Correlated Multi-Product Dynamic Pricing and Replenishment}
In this section, we explore the correlated multi-product dynamic pricing and replenishment problem, where the demand for one product is influenced by the pricing and inventory levels of others. Given these interdependencies, it is crucial to consider all products' pricing and inventory simultaneously when making decisions for any single product.

\subsubsection{Markov Game Formulation}
MG formulation in this section is generally consistent with Section 3.4.1, with the following differences: (i) The number of agents $N \in \mathbb{N}^+ $ equals the product of the number of products and decisions, with the latter being two, specifically pricing and replenishment decisions. (ii) The joint reward $r_t = \mathcal{R}(s_t|\bold{a}_t)$ is defined as the total sales profit of all products in period $t$, given by
\begin{equation}\label{eq:7}
    r_t = \sum_{i \in N} (p^i_t U^i_t - h^i I^i_t - b^i L^i_t - c^i q^i_t), 
\end{equation}
where the superscript $i$ represents  product $i$. For example, $p^i_t$ represents the price offered by the agent for product $i$ in period $t$.

\subsubsection{Experimental Results}
Based on cosine similarity calculations, we identify 3 products with correlated demand from the original data and proceed to apply FSDA to optimize both their pricing and replenishment strategies simultaneously. We compare the performance of FSDA and four heuristic methods (BSLP, (s,S,p), Myopic and ‘Subs’) in the context of correlated MPDPR problems. BSLP, (s,S,p) and Myopic are the same as Section 3.7. Since these strategies only consider a single product, during the experiments, we can only optimize multiple products separately using these three strategies. The ‘Subs’ heuristic method, developed by \citet{song2023dynamic}, optimally handles dynamic pricing and inventory decisions for substitutable products whose demands depend on their prices. In each period, ‘Subs’ policy first
pinpoints products that are overstocked and thus require no replenishment. It then specifies
the optimal base-stock levels for the understocked products and the optimal market shares (and
prices) for all products, contingent on the overstock levels. Since the correlated multi-product dynamic pricing and replenishment  problem considered in this section involves not only substitutable products (negative demand correlation) like SKU1 and SKU2, but also complementary products (positive demand correlation) like SKU2 and SKU3, the ‘Subs’ is only applied to the joint decisions of SKU1 and SKU2 when compared with our method.
 The experimental comparison results are shown in Figure \ref{fig:FSDA_4baselines}. \par

\begin{figure}[h!]
  \centering
  \caption{Comparison of reward under correlated multi-product system}
  \includegraphics[width=0.7\textwidth]{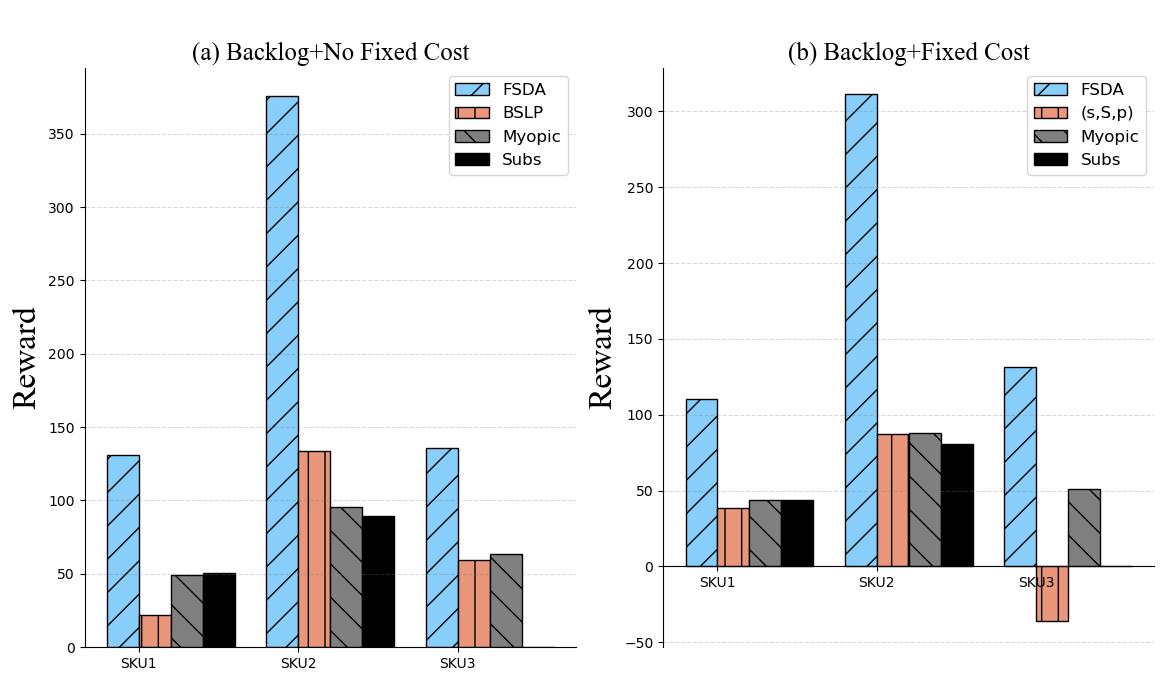} 
  
  \label{fig:FSDA_4baselines}
\end{figure}
Figure \ref{fig:FSDA_4baselines} shows that FSDA achieves the highest sales profits in scenarios with and without fixed costs, significantly outperforming other strategies. Specifically, without considering fixed ordering costs, FSDA increases total profit of all products by 198.35\% compared to BSLP, 208.17\% compared to Myopic, and 261.59\% compared to Subs. With fixed ordering costs included, FSDA outperforms (s,S,p) by 514.32\%, Myopic by 201.94\%, and Subs by 238.98\%. The advantage of the Myopic policy considering the impact of next period's purchase costs is also no longer apparent. The (s,S,p) policy, which focuses on reducing ordering frequency, fails to adapt to the demand fluctuations caused by product interdependencies and other factors in the market, resulting in poor performance in SKU3. Additionally, although the ‘Subs’ policy is optimal for substitutable products whose demands depend on their prices, it nearly fails in this multi-product system where demand is also influenced by many other factors. 

\section{Conclusion}
This paper studies the dynamic pricing and replenishment problems. A two-timescale stochastic approximation scheme is proposed to address the problem of joint pricing and replenishment decisions. Specifically, we first investigate a single-product dynamic pricing and replenishment problem, taking into account competitor strategies and market conditions. Subsequently, we enhance our model by incorporating a decision tree-based ML model for demand forecasting. Numerical results for both single product and multiple correlated products demonstrate the effectiveness of our approach. We find that in joint pricing and replenishment decisions, focusing solely on inventory levels and the impact of price on demand may not necessarily lead to profit improvement. While these factors are crucial, they offer a limited perspective when applied in practice. Our findings suggest that profit optimization requires a more holistic approach, one that integrates a broader range of market dynamics. Factors such as customer behavior, competitor strategies, and sales timing all play critical roles in optimizing pricing and replenishment strategies. By accounting for these multiple market influences, companies can not only improve the accuracy of demand forecasting but also enhance their ability to align pricing with supply chain operations, ultimately driving greater profitability and operational efficiency. \par
In principle, our method can be applied to any scenario where multiple decisions interact but occur at different frequencies. There are several possible extensions of our current work. For example, the present research involves a relatively small scale of multi-product problems; in future studies, expanding the scale would be crucial for e-commerce retailers who currently need to make decisions for many SKUs simultaneously.

\bibliography{main}


%
%
%



\newpage
\section*{\centering Appendix}
\setcounter{equation}{0}
\renewcommand\theequation{A.\arabic{equation}}

\section*{A. Proof of Theoretical Results}  \label{appendix:A} 
\subsection*{Proof of Proposition 1}
\begin{proof}{Proof}
    We need to prove that given a function $\lambda(p)$ and Poisson distribution $D(p) \sim Pois(\lambda(p))$, there does not exist any function $\gamma(p)$, $\delta(p)$ and random variable $\epsilon$ which is independent with $p$ satisfying
    \begin{equation} \small \label{pro7}
        D(p) \overset{d}{=} \gamma(p)\epsilon + \delta(p),
    \end{equation}
    where $\overset{d}{=}$ means two random variables are equal in distribution. If there exists $\gamma(p)$, $\delta(p)$ and random variable $\epsilon$ satisfying the conditions, we can take expectation and variance of the random variable on both sides of Eq.(\ref{pro7}):
    $\lambda(p) = \gamma(p)\mathbb{E}[\epsilon] + \delta(p)$, $\lambda(p) = \gamma(p)^2 Var(\epsilon).$
    So we can get the solution
    $\gamma(p) = \sqrt{\frac{\lambda(p)}{Var(\epsilon)}}, \delta(p) = \lambda(p) - \sqrt{\frac{\lambda(p)}{Var(\epsilon)}} \mathbb{E}[\epsilon].$
    Furthermore, substituting the results into Eq.(\ref{pro7}), we can get  $D(p) \overset{d}{=} \sqrt{\frac{\lambda(p)}{Var(\epsilon)}}(\epsilon-\mathbb{E}[\epsilon]) + \lambda(p).$
    Then, taking the third-order central moments: 
    $\lambda(p) = \mathbb{E}[D(p)-\lambda(p)]^3 =  \bigg(\frac{\lambda(p)}{Var(\epsilon)}\bigg)^{\frac{3}{2}}\mathbb{E}[(\epsilon-\mathbb{E}[\epsilon])^3], $
    we can conclude from this equation that $\lambda(p)=Var(\epsilon)^3/\mathbb{E}[(\epsilon-\mathbb{E}[\epsilon])^3]^2$ is a constant independent of $p$, which is a contradiction since $\lambda(p)$ should be a function of $p$. Therefore, this completes the proof of the proposition. \Halmos
    
\end{proof}
\subsection*{Proof of Lemma 1}
\begin{proof}{Proof}
  We have
    $\frac{\partial \lambda}{\partial p} = \frac{\eta \Delta le^{lp+a}}{(1+e^{lp+a})^2} <0$,
    $\frac{\partial^2 \lambda}{\partial p^2} = \frac{\eta \Delta l^2e^{lp+a}(1-e^{lp+a})}{(1+e^{lp+a})^3} < 0$.
    The latter inequality uses the assumption that $lp+a > 0$, which ends the proof. \Halmos
\end{proof}

\subsection*{Proof of Proposition 3}
\begin{proof}{Proof}
We rewrite the revenue function by the definition of expectation and calculate its derivative as follows:
  \begin{equation*}\small
        \begin{aligned}
        E(p,x) = p\mathbb{P}(d>x)x + p\mathbb{E}[d\mathbb{I}_{\{d\le x\}}] = (1-\sum_{k=0}^{\lfloor x \rfloor}\frac{\lambda^ke^{-\lambda}}{k!})xp + \sum_{k=0}^{\lfloor x \rfloor}\frac{\lambda^ke^{-\lambda}}{k!}kp,
        \end{aligned}
    \end{equation*}
\begin{equation*}\small
        \begin{aligned}
            \frac{\partial E(p,x)}{\partial p} = (1-\sum_{k=0}^{\lfloor x \rfloor}\frac{\lambda^ke^{-\lambda}}{k!})x + xp(-\sum_{k=0}^{\lfloor x \rfloor}\frac{\lambda^{k-1}e^{-\lambda}}{k!}(k-\lambda)\frac{\partial \lambda}{\partial p}) + \sum_{k=0}^{\lfloor x \rfloor}\frac{\lambda^ke^{-\lambda}}{k!}k + p\sum_{k=0}^{\lfloor x \rfloor}\frac{\lambda^{k-1}e^{-\lambda}}{k!}(k-\lambda)k\frac{\partial \lambda}{\partial p}.
        \end{aligned}
    \end{equation*}
Calculate the second order derivative and we can get
\begin{equation*}\small
        \begin{aligned}
            \frac{\partial^2 E(p,x)}{\partial p^2} =& 2 \frac{\partial \lambda}{\partial p}\sum_{k=0}^{\lfloor x \rfloor}\frac{\lambda^{k-1} e^{-\lambda}}{k!}(k-x)(k-\lambda)  + p(\frac{\partial \lambda}{\partial p })^2\sum_{k=0}^{\lfloor x \rfloor}\frac{\lambda^{k-2} e^{-\lambda}}{k!}(k-x)((k-\lambda)^2-k)\\ +& p\frac{\partial^2 \lambda}{\partial p^2}\sum_{k=0}^{\lfloor x \rfloor}\frac{\lambda^{k-1} e^{-\lambda}}{k!}(k-x)(k-\lambda).
        \end{aligned}
    \end{equation*}
By the same analysis as Proposition \ref{proposition:2}, we can reach the conclusion that the sum of the three terms is negative, so the revenue function is concave with respect to $p$.  \Halmos
\end{proof}

\subsection*{Proof of Lemma 2}
\begin{proof}{Proof}
Note that $L(p,x) = \mathbb{E}[(x-d)^+] + \mathbb{E}[d-x] = I(p,x) + \hat{\lambda} - x$.
By Proposition \ref{proposition:2} and Assumption \ref{assumption:3}, $I(p,x)$ is convex and $\hat{\lambda}(p)$ is linear, so 
       $ \frac{\partial^2 L(p,x)}{\partial p^2} = \frac{\partial^2 I(p,x)}{\partial p^2} > 0.$
So $L(p,x)$ is convex with respect to $p$. \Halmos
\end{proof}

\subsection*{Proof of Theorem 1}
\begin{proof}{Proof}
    Since $h$, $b$, $c$ are positive, the concavity of $E(p,x)$, $-hI(p,x)$, and $-bL(p,x)$ with respect to $p$ is guaranteed given $x$. The theorem holds because the sum of concave functions remains concave. The proof of non-joint concavity is the same as Remark \ref{remark:1}. Furthermore, plug in $E(p,x)$, $I(p,x)$, $L(p,x)$ and we can rewrite $F(p,x)$ as 
\begin{equation} \small
    \begin{aligned} \label{eq:16}
 F(p, x) =&  p\mathbb{E}[\min(d,x)] - (h+b)\sum_{k=0}^{\lfloor x \rfloor}\frac{\lambda^ke^{-\lambda}}{k!}(x-k) -b\lambda + bx - c(x-x_0) \\
 =& \bigg(b-c+p - \sum_{k=0}^{\lfloor x \rfloor}\frac{\lambda^ke^{-\lambda}}{k!}(h+b+p) \bigg)x + \sum_{k=0}^{\lfloor x \rfloor}\frac{\lambda^ke^{-\lambda}k}{k!}(p+h+b)-b\lambda+cx_0,
    \end{aligned}
\end{equation}
where $\lambda(p) = \eta\Delta e^a(1+lp)$. The first equation of conditions (\ref{eq:15}) is naturally derived by the first order condition. For the other two equations, it is obvious that $F(p,x)$  is continuous and piecewise linear with respect to $x$, with integer points serving as nodes. This implies that $F(p,x)$ reaches its maximum at certain integer nodes, so we only need to check the values at integer nodes. According to Eq.(\ref{eq:16}), assuming $n$ is an integer and $x \in [n-1,n)$, the slope of $F(p,x)$ with respect to $x$ is given by $b-c+p-(h+b+p)\sum_{k=0}^{n-1}\frac{\lambda^k(p)e^{-\lambda(p)}}{k!}. $
When $x \in [n,n+1) $, the slope is $b-c+p-(h+b+p)\sum_{k=0}^{n}\frac{\lambda^k(p)e^{-\lambda(p)}}{k!},$ with a negative increment $-(h+b+p)\frac{\lambda^n(p)e^{-\lambda(p)}}{n!}$. Therefore, the slope is decreasing, which means $F(p,x)$ is concave with respect to $x$. Consequently, the last two equations of condition (\ref{eq:15}) hold naturally if $F(p, x)$ does not take its maximum at the boundary. \Halmos
\end{proof}

\subsection*{Proof of Proposition 4}
\begin{proof}{Proof}
    Firstly, we show that under Assumptions \ref{asmp:6}(i) and \ref{asmp:7}(ii), the conclusion will be violated. We use the additive model, where expected demand is a linear function of price, as a counterexample. Similar calculations apply to other cases as well. Set $d(p,\epsilon) = \epsilon + d(p)$,  where $d(p) = lp + a$ and the probability density function of $\epsilon$ is $f$. Then the revenue function is
    \begin{equation*}\small
        E(p,x) = p\mathbb{E}[\min(\epsilon + d(p), x)] = p\bigg(\int_{-\infty}^{x-d(p)}(u+d(p))f(u)du + x\int_{x-d(p)}^{\infty}f(u)du\bigg).
    \end{equation*}
    Set $A=pf(x-d(p))>0$, $B=\int_{x-d(p)}^{\infty}f(u)du>0$, and  $l=\frac{\partial d(p)}{\partial p}<0$. The second derivative of this function can be obtained by calculation:
        $\frac{\partial^2 E(p,x)}{\partial p^2} = 2\int_{-\infty}^{x-d(p)}\frac{\partial d(p)}{\partial p}f(u)du - pf(x-d(p))\bigg(\frac{\partial d(p)}{\partial p}\bigg)^2=2l(1-B)-l^2A$, $
        \frac{\partial^2 E(p,x)}{\partial p\partial x} = \int_{x-d(p)}^{\infty}f(u)du + pf(x-d(p))\frac{\partial d(p)}{\partial p}=Al+B$, 
        $\frac{\partial^2 E(p,x)}{\partial x^2} = -pf(x-d(p)) = -A$.
    The determinant of the Hessian matrix is $-2Al-B^2$ and the sign is uncertain. In the region where $-2Al-B^2<0$ the revenue function is not jointly concave. Therefore, the joint concavity of $F(p,x)$ can be lost. Under Assumptions \ref{asmp:6}(ii) and Assumptions \ref{asmp:7}(i), the proof of non-joint concavity is the same as Remark \ref{remark:1}. This is also because the inventory function as well as profit function is piecewise linear with respect to x when the demand follows a Poisson distribution. \Halmos
\end{proof}

\subsection*{Proof of Proposition 5}
\begin{proof}{Proof}
    We can calculate the derivative of $F(p,x)$ with respect to $p$
    \begin{equation} \small
    \begin{aligned} \label{eq:17}
 \frac{\partial F(p, x)}{\partial p} =&  \frac{\partial \lambda}{\partial p} \bigg(\sum_{k=0}^{\lfloor x \rfloor} \frac{\lambda^{k-1}e^{-\lambda}}{k!}(k-\lambda)(k-x)(p+h+b) \bigg) -b\frac{\partial \lambda}{\partial p} +x + \sum_{k=0}^{\lfloor x \rfloor} \frac{\lambda^{k}e^{-\lambda}}{k!}(k-x).
    \end{aligned}
\end{equation}
It is also a piecewise linear function with respect to $x$ and the slope can be considered as $\frac{\partial^2 F(p,x)}{\partial x \partial p}$, which is
\vspace{-0.5cm}
\begin{equation}\small \label{eq:18}
    \begin{aligned} 
 1 + \sum_{k=0}^{\lfloor x \rfloor} \frac{\lambda^{k-1}e^{-\lambda}}{k!}(-\lambda - \frac{\partial \lambda}{\partial p}(k-\lambda)(p+h+b)).
    \end{aligned}
\end{equation}
Note that when $x$ tends to infinity, the value tends to 0:
\begin{equation*} \small
    \begin{aligned} 
  1 + \sum_{k=0}^{\infty} \frac{\lambda^{k-1}e^{-\lambda}}{k!}(-\lambda - \frac{\partial \lambda}{\partial p}(k-\lambda)(p+h+b))= 1 - \sum_{k=0}^{\infty} \frac{\lambda^{k}e^{-\lambda}}{k!} - (p+h+b)\frac{\partial \lambda}{\partial p}\sum_{k=0}^{\infty} \frac{\lambda^{k-1}e^{-\lambda}}{k!} (k-\lambda) = 1 -1 -0=0.
    \end{aligned}
\end{equation*}

Besides, when $x$ is sufficiently large, $-\lambda - \frac{\partial \lambda}{\partial p}(k-\lambda)(p+h+b)$ turns positive due to $\frac{\partial \lambda}{\partial p}$ being negative. Consequently, the function in Eq.(\ref{eq:18}) decreases when $x$ is small and increases when $x$ is large. Since its limit is zero, the sign of Eq.(\ref{eq:18}) is negative for large $x$. Whether Eq.(\ref{eq:18}) can be positive or not depends on the sign when $x=0$.
\par
Furthermore, we can analyze the monotonicity of $p(x)$. When $x=0$, $F(p,x)=-b\lambda(p)+cx_0$, which takes its maximum when $p$ reaches its upper bound. Then, as $x$ increases, the proposition holds if $p$ always reaches its upper bound. If $p$ begins to decrease, this indicates that the root of Eq.(\ref{eq:17}) has been taken, causing Eq.(\ref{eq:18}) to become negative. Furthermore, $\frac{\partial F}{\partial p}$ decreases as $x$ increases. Since $F(p,x)$ is concave with respect to $p$, $\frac{\partial F}{\partial p}$ decreases with respect to $p$. To ensure that $\frac{\partial F}{\partial p} = 0$ still holds, $p$ should decrease as $x$ increases. The proof above is clear if we take the perspective of the implicit function theorem: $\frac{\partial p}{\partial x} = - \frac{\partial^2 F}{\partial x \partial p}\bigg/\frac{\partial^2 F}{\partial p^2}<0,$ since both the numerator and the denominator are negative. The reason for expressing it as above is that $x$ is non-differentiable at integer values. \Halmos
\end{proof}
\subsection*{Proof of Proposition 6}
\begin{proof}{Proof}
To prove this proposition, we employ two classical stochastic gradient estimation methods: the infinitesimal perturbation analysis (IPA) method and the likelihood ratio (LR) method, as outlined in \cite{FU2006575}. For discrete distributions, LR works if the parameter occurs in the probability density function, whereas IPA works if the parameter serves as a structural parameter. In our case, $p$ is a distributional parameter that only occurs in the distribution of $d$. Therefore, LR method can be implemented when estimating the gradient with respect to $p$, as calculated below: 
\begin{equation*} \small
    \begin{aligned} 
 g(p,x) &= \frac{\partial}{\partial p}\bigg(p\mathbb{E}[\min(d,x)] - (h+b)\mathbb{E}[(x-d)^+] -b\lambda + bx - c(x-x_0)\bigg)\\ 
 &= \mathbb{E}[\min(d,x)] + p\frac{\partial}{\partial p}\sum_{d}\mathbb{P}(d)\min(d,x) - (h+b)\frac{\partial}{\partial p}\sum_{d}\mathbb{P}(d)(x-d)^+-b\frac{\partial \lambda}{\partial p }\\
 &= \mathbb{E}[\min(d,x)] + p\sum_{d}\frac{\partial}{\partial p} \mathbb{P}(d)\min(d,x) - (h+b)\sum_{d}\frac{\partial}{\partial p} \mathbb{P}(d)(x-d)^+-b\frac{\partial \lambda}{\partial p }\\
 &= \mathbb{E}[\min(d,x)] + p\sum_{d}\mathbb{P}(d) \frac{\partial \log\mathbb{P}(d)}{\partial p}\min(d,x)- (h+b)\sum_{d}\mathbb{P}(d) \frac{\partial \log\mathbb{P}(d)}{\partial p}(x-d)^+-b\frac{\partial \lambda}{\partial p }\\
 &= \mathbb{E}[\min(d,x)] + p\mathbb{E}\bigg[\frac{\partial \log\mathbb{P}(d)}{\partial p} \min(d,x)\bigg] -(h+b)\mathbb{E}\bigg[\frac{\partial \log \mathbb{P}(d)}{\partial p} (x-d)^+\bigg] -b\frac{\partial \lambda}{\partial p }.
      \end{aligned}
\end{equation*}
   The interchange of taking summations and taking derivatives in the third equality comes from the uniform convergence of the function term series $\sum_{d}\frac{\partial }{\partial p}\mathbb{P}(d)\min(d,x)$ and $\sum_{d}\frac{\partial}{\partial p} \mathbb{P}(d)(x-d)^+$ \citep[Theorem 7.17]{rudin1964principles}. Meanwhile, $x$ serves as a structural parameter in the following expectations; thus, we can apply the IPA method for estimating the gradient with respect to $x$, as calculated below: 
 \begin{equation*} \small
    \begin{aligned}
h(p,x) &= \frac{\partial}{\partial x}\bigg(p\mathbb{E}[\min(d,x)] - (h+b)\mathbb{E}[(x-d)^+] -b\lambda + bx - c(x-x_0)\bigg)\\
&= p\frac{\partial}{\partial x}\sum_{d}\mathbb{P}(d)\min(d,x) - (h+b)\frac{\partial}{\partial x}\sum_{d}\mathbb{P}(d)(x-d)^+ + b-c\\
&= p\sum_{d}\mathbb{P}(d)\frac{\partial}{\partial x}\min(d,x) - (h+b)\sum_{d}\mathbb{P}(d)\frac{\partial}{\partial x}(x-d)^+ + b-c\\
&= b-c+p - (h+b+p)\mathbb{E}[1_{d \le x}].
    \end{aligned}
\end{equation*}
The interchange of taking summations and taking derivatives in the third equality comes from the uniform convergence of the function term series $\sum_{d}\mathbb{P}(d)\frac{\partial}{\partial x}\min(d,x)$ and $\sum_{d}\mathbb{P}(d)\frac{\partial}{\partial x}(x-d)^+$. \Halmos
\end{proof}
\subsection*{Proof of Lemma 3}
\begin{proof}{Proof}
    Define $M_{k}^{(1)} = \hat{g}(p_k,x_k)- \mathbb{E}[\hat{g}(p_k,x_k)|\mathcal{F}_k]$, $M_{k}^{(2)} = \hat{h}(p_k,x_k)- \mathbb{E}[\hat{h}(p_k,x_k)|\mathcal{F}_k]$, then $M_{k}^{(1)}$ and $M_{k}^{(2)}$ are martingale difference sequences with respect to the increasing $\sigma$-fields, which means $\mathbb{E}[M_{k}^{(1)}|\mathcal{F}_{k-1}] = \mathbb{E}[M_{k}^{(2)}|\mathcal{F}_{k-1}] = 0$ for every $k$. Due to the finite second-order moments of the Poisson distribution, it is easy to check that 
        $\mathbb{E}[| M_{k}^{(1)}|^2|\mathcal{F}_{k-1}]\le K(1+\vert p_k\vert^2 +|x_k|^2)$, $\mathbb{E}[|M_{k}^{(2)}|^2|\mathcal{F}_{k-1}]\le K(1+|p_k|^2 +|x_k|^2).$
    Rewrite the iteration process as
    \begin{equation*}\small
    \begin{aligned} 
p_{k+1} &= p_k + \alpha_k\hat{g}(p_k,x_k) = p_k + \alpha_k (\mathbb{E}[\hat{g}(p_k,x_k)|\mathcal{F}_k] + M_{k}^{(1)}) = p_k + \alpha_k (g(p_k,x_k) + M_{k}^{(1)}), \\
x_{k+1}&= x_k+ \beta_k\hat{h}(p_k,x_k) = x_k + \beta_k(\mathbb{E}[\hat{h}(p_k,x_k)|\mathcal{F}_k] + M_{k}^{(2)}) = x_k + \alpha_k(\frac{\beta_k}{\alpha_k}h(p_k,x_k) + \frac{\beta_k}{\alpha_k}M_{k}^{(2)}).
    \end{aligned}
\end{equation*}
Since $g(p,x) =  \frac{\partial \lambda}{\partial p} \bigg(\sum_{k=0}^{\lfloor x \rfloor} \frac{\lambda(p)^{k-1}e^{-\lambda(p)}}{k!}(k-\lambda(p))(k-x)(p+h+b) \bigg) -b\frac{\partial \lambda}{\partial p} +x + \sum_{k=0}^{\lfloor x\rfloor} \frac{\lambda(p)^{k}e^{-\lambda(p)}}{k!}(k-x)$,
the partial derivative $\frac{\partial g(p,x)}{\partial p}$ is continuous in a bounded interval. So $\frac{\partial g(p,x)}{\partial p}$ is bounded, which implies $g(p,x)$ is Lipschitz continuous with respect to $p$. $g(p,x)$ is piecewise linear with respect to $x$ so it is also Lipschitz continuous with respect to $x$. Under the faster time scale, $x_k$ can be seen as static when $k$ is large enough due to $\frac{\beta_k}{\alpha_k} \rightarrow 0$. It is enlightening to compare the coupled iterations Eq.(\ref{eq:19}) to the ordinary differential equations (ODEs):
\begin{equation} \small \label{eq:20}
\dot{p}(t) =  g(p(t),x(t)),
\end{equation}
\vspace{-26pt} 
\begin{equation} \small \label{eq:21}
\dot{x}(t)=\gamma h(p(t),x(t)),
\end{equation}
in the limit $\gamma \rightarrow 0$ given $\frac{\beta_k}{\alpha_k} \rightarrow 0$, which means $p(t)$ is a fast transient compared to $x(t)$. It makes sense to consider $x(t)$ as quasi-static while analyzing the behavior of $p(t)$. This suggests focusing on ODE Eq.(\ref{eq:20}), $\dot{p}(t) = g(p(t),x),$ where $x$ is held fixed as a constant parameter. By Theorem \ref{theorem:1}, $F(p,x)$ is concave with respect to $p$, so $g(p,x)$ has a unique root if $x$ is given. Therefore, $p^{\star}(x)$ is well-defined and ODE Eq.(\ref{eq:20}) has a globally asymptotically stable equilibrium $p^{\star}(x)$. By Theorem 2 in \citet[Chap.2]{borkar2009stochastic} and Lemma 1 in \citet[Chap.6]{borkar2009stochastic}, the sequence $\{(p_k,x_k)\}$ satisfies all the conditions required by the theorems and converges to the limit point of the ODE $\dot{p}(t) = g(p(t), x(t))$, $\dot{x}(t) = 0$. The limit point of the ODE satisfies $g(p^{\star}(x),x)=0$, so the sequence $\{(p_k,x_k)\}$ generated by the iteration Eq.(\ref{eq:19}) will converge to the set $\{(p^{\star}(x),x)\}$, which ends the proof.
\Halmos
\end{proof}
\subsection*{Proof of Theorem 3}
\begin{proof}{Proof}
    Define $\zeta_k = p_k-p^{\star}(x_k)$, then we can write the iteration for $p_k$ in terms of $\zeta_k$ as
     \begin{equation*} \small
     \zeta_{k+1} = \zeta_k +\alpha_k\hat{g}(p_k,x_k) + p^{\star}(x_k) - p^{\star}(x_{k+1}).
     \end{equation*}
    Denote the Lipschitz constant of $p^{\star}(x)$ by $L_p$, then we have
     \begin{equation*} \small
     |p^{\star}(x_k) - p^{\star}(x_{k+1})| \le L_p|x_k-x_{k+1}|\le L_p\beta_k|\hat{h}(p_k,x_k)| \le L_p\beta_k(b-c+p_k) \le \beta_kL_pC_h,
     \end{equation*}
    where $C_h$ is the bound of the $\hat{h}(p_k,x_k)$. Also, denote the bound of the $\hat{g}(p_k,x_k)$ by $C_g$, then we can square both sides of the equation for $\zeta$ using the above inequality:
    \begin{equation*}\small
        \begin{aligned}
            |\zeta_{k+1}|^2 \le |\zeta_{k}|^2  + \alpha_k^2C_g^2 + \beta_k^2L_p^2C_h^2 + 2\alpha_k\beta_kL_pC_hC_g + 2\alpha_k\zeta_k\hat{g}(p_k,x_k) + 2\beta_kL_pC_h|\zeta_k|.
        \end{aligned}
    \end{equation*}
    By Taylor's expansion of $g(p,x)$:
     $g(p_k,x_k) = g(p_k,x_k) - g(p^{\star}(x_k),x_k) = \frac{\partial g}{\partial p}(\Tilde{p}_k,x_k)(p_k- p^{\star}(x_k)) = \frac{\partial g}{\partial p}(\Tilde{p}_k,x_k)\zeta_k$,
    where $\Tilde{p}_k$ lies between $p_k$ and $p^{\star}(x_k)$.
    According to Theorem \ref{theorem:1}, $\frac{\partial g}{\partial p}$ is negative. Due to the continuity of $\frac{\partial g}{\partial p}$, we can assume that $\frac{\partial g}{\partial p} \le -\epsilon \le 0$ in the feasible region of $p$. And note that $\mathbb{E}[\hat{g}(p_k,x_k)|\mathcal{F}_k] = g(p_k,x_k)$, we can take conditional expectation on both sides of the above inequality for $\zeta$:
    \begin{equation*}\small
        \begin{aligned}
            \mathbb{E}[|\zeta_{k+1}|^2|\mathcal{F}_k] \le& |\zeta_{k}|^2  + \alpha_k^2C_g^2 + \beta_k^2L_p^2C_h^2 + 2\alpha_k\beta_kL_pC_hC_g + 2\alpha_k\zeta_k\mathbb{E}[\hat{g}(p_k,x_k)|\mathcal{F}_k] + 2\beta_kL_pC_h|\zeta_k|\\
            =& |\zeta_{k}|^2  + \alpha_k^2C_g^2 + \beta_k^2L_p^2C_h^2 + 2\alpha_k\beta_kL_pC_hC_g + 2\alpha_k\zeta_k g(p_k,x_k) + 2\beta_kL_pC_h|\zeta_k| \\
            \le& (1-2\epsilon\alpha_k)|\zeta_{k}|^2  + \alpha_k^2C_g^2 + \beta_k^2L_p^2C_h^2 + 2\alpha_k\beta_kL_pC_hC_g + 2\beta_kL_pC_h|\zeta_k|.
        \end{aligned}
    \end{equation*}
    We assume, without loss of generality, that $\epsilon$ is sufficiently small such that $1-2\epsilon\zeta_k > 0$ for all $k$. Next, by taking the expectation again and applying the Cauchy Schwarz inequality, we further obtain
    \begin{equation*}\small
    \begin{aligned}
        \mathbb{E}[|\zeta_{k+1}|^2] \le& (1-2\epsilon\alpha_k) \mathbb{E}[|\zeta_{k}|^2] + \alpha_k^2C_g^2 + \beta_k^2L_p^2C_h^2 + 2\alpha_k\beta_kL_pC_hC_g + 2\beta_kL_pC_h\mathbb{E}[|\zeta_{k}|] \\
        \le& (1-2\epsilon\alpha_k) \mathbb{E}[|\zeta_{k}|^2] + 2\beta_kL_pC_h\sqrt{\mathbb{E}[|\zeta_{k}|^2]} + \alpha_k^2C_g^2 + \beta_k^2L_p^2C_h^2 + 2\alpha_k\beta_kL_pC_hC_g \\
        \le & \bigg(\sqrt{1-2\epsilon\alpha_k}\sqrt{\mathbb{E}[|\zeta_{k}|^2]} + \frac{\beta_kL_pC_h}{\sqrt{1-2\epsilon\alpha_k}}\bigg)^2+\alpha_k^2C_g^2 + 2\alpha_k\beta_kL_pC_hC_g \\
        \le & \bigg((1-\alpha_k\epsilon)\sqrt{\mathbb{E}[|\zeta_{k}|^2]} + \frac{\beta_kL_pC_h}{\sqrt{1-2\epsilon\alpha_0}}\bigg)^2+\alpha_k^2C_g^2 + 2\alpha_k\beta_kL_pC_hC_g,
    \end{aligned}
    \end{equation*}
    where the last inequality is due to $\sqrt{1-x}\le 1-x/2$ for $x\in [0,1]$.
    
    Define the mapping
    \begin{equation*}\small
    T_k(x) \doteq \sqrt{\bigg((1-\alpha_k\epsilon)x + \frac{\beta_kL_pC_h}{\sqrt{1-2\epsilon\alpha_k}}\bigg)^2+\alpha_k^2C_g^2 + 2\alpha_k\beta_kL_pC_hC_g}
    \end{equation*}
    and consider the sequence of real numbers $\{x_k\}$ generated by $x_{k+1} = T_k(x_k)$ for all $k$ with $x_0 \doteq \sqrt{\mathbb{E}[\zeta_0^2]}.$ A simple inductive argument shows that $\sqrt{\mathbb{E}[\zeta_k^2]\le x_k}$ for all $k$. In addition, notice that $T_k(x)$ is a function of the form $h(x) = \sqrt{x^2+b^2}, x>0$ for some constant $b \neq 0$. Since $\frac{dh(x)}{dx} = \frac{x}{\sqrt{x^2+b^2}} < 1$, we have $|h(x)-h(y)|\le|x-y|$ for all $x,y>0$. This implies that $T_k(x)$ is a contraction mapping satisfying
    $|T_k(x)-T_k(y)|\le (1-\alpha_k\epsilon)|x-y|.$

    The unique fixed point $x_k^{\star}$ of $T_k$ can be obtained by solving the quadratic equation $T_k(x) = x$, and written in the form:
    \begin{equation*}\small
        x_k^{\star} = O(\frac{\beta_k}{\alpha_k}) + O(\alpha_k^{\frac{1}{2}}) + \text{higher-order} \  \text{terms}.
    \end{equation*}
    By applying the same technique as in \cite{hu2023quantile} and \cite{jiang2023quantile}, we can reach the same conclusion of $\zeta_k$. \Halmos
\end{proof}
\subsection*{Proof of Proposition 7}
\begin{proof}{Proof}
Let the expected future discounted profits generated in period $t$ be $V_t(s_t)$:
\begin{equation}\small
    V_t(s_t) = \max \limits_{p_t, q_t} \Big\{p_t S_t - h I_t - b L_t - c q_t + \gamma \sum_{s_{t+1}} \mathcal{P}(s_{t+1}|s_t,a_t) V_{t+1}(s_{t+1}) \Big \}.
\end{equation}
 
For each period $t$, suppose the stochastic demand $d_t$ has $D$ possible values, the computation effort needed to obtain the current profit term $\max \limits_{p_t, q_t}  (p_t S_t - h I_t - b L_t - c q_t)$ is denoted by $E_t$, and the computational time $C_t$ needed to get the overall profit $V_t(s_t)$ is 
\begin{equation} \small \label{eq:A1.1}
    C_t = PQ[E_t + D C_{t+1}].
\end{equation}

The amount of computation needed to obtain $E_t$ is $D$ since $S_t$, $I_t$ and $L_t$ are derived from demand $d_t$. For each period $t$, the value of $D$ depends on the number of regressors $K$ and the dimension of each regressor $R_k \in [R_{min},R_{max}]$. Hence, we have
$E_t = D = \prod_{i=1}^{K} R_k$.
Since the terminal profit is 0, i.e., $V_{T+1}(\cdot)$ = 0, we have $C_{T+1}$ = 0, which gives $C_T = PQ E_{T}$. With Eq. (\ref{eq:A1.1}), we have 
\begin{align}\small
    \begin{split}
     C_T &= PQ E_{T}, \\
    C_{T-1} &= PQ E_{T-1} + (PQ)^2 (D E_T),\\
    & \vdots \\
    C_1 &= PQ E_1 + (PQ)^2 (D E_{2}) + (PQ)^3 (D)^2 E_{3} \\
    &+ ... + (PQ)^T (D)^{T-1} E_T \\
    &=\sum_{t=1}^T \Big[(PQ)^t (D)^{t-1} E_{t} \Big] \\
    &=\sum_{t=1}^T \Big[(PQ D)^{t}\Big]. \\
    \end{split}
\end{align}

With $R_{min} \leq R_k \leq R_{max}$, we have $(R_{min})^K \leq D \leq (R_{max})^K$ and 
\begin{equation*}\small
    \begin{aligned}
        C_1 &\geq  \sum_{t=1}^T \Big[(PQ)^t (R_{min})^{Kt}\Big] = O(P^T Q^T (R_{min})^{KT}), \\
        C_1 &\leq \sum_{t=1}^T \Big[(PQ)^t (R_{max})^{Kt}\Big] = O(P^T Q^T (R_{max})^{KT}).
    \end{aligned}
\end{equation*}

Hence,  the lower bound of the total computational effort is $O(P^T Q^T (R_{min})^{KT})$ and the upper bound is $O(P^T Q^T (R_{max})^{KT})$. \Halmos
\end{proof}

\subsection*{Proof of Proposition 8}
\begin{proof}{Proof}
With the back-propagation (BP) algorithm, the computational complexity of the parameter update phase is of the same order as that of model inference during the simulation phase. Both of these are dominated by the computational cost of passing the product state $s_t^i$ through the neural network to obtain the product action $a_t^i$. We assume that the computational cost of inputting the product state $s_t^i$ to yield the product action $a_t^i$ is $C$, which is determined by the computational cost of passing one state $s_t^i$ through each module of the neural network shown in Figure \ref{fig:Network}. For MLP1, the input dimension is equal to the state space dimension, which is $D_I$. There are two hidden layers with a width of $W_1$ each. Therefore, the computation of this part can be expressed as $(D_I + W_1) * W_1$. In the case of the GRU, the computation with the highest order occurs in the calculation of the candidate hidden state, as mentioned in \citet{dey2017gate}. For GRU1, the computation is $(W_1 + W_2)W_2$, while for GRU2, it is $2W_2*W_2$. For MLP2, with the dimension of the input as $W_2$ and the dimension of its output as $D_A$, the computation that happens here is $W_2 D_A$. Due to the linear relationship between the computational complexity of the Softmax operation and the input dimension, it can be safely disregarded. By adding the computation cost of each module in the neural network together, we have
\begin{equation*}\small 
    C = O\big((D_I + W_1)W_1 + (W_1 + 3 W_2)W_2 + W_2 D_A \big).
\end{equation*}
Considering the number of products $N$, the number of training episodes $M$, and the number of periods $T$ during each episode, we can derive the computation complexity of the training procedure as
\begin{equation} \label{eqn:A.2}
    O\Big(\big((D_I+W_1) W_1+(W_1+3W_2)W_2+ W_2 D_A \big)NMT \Big).
\end{equation}
The computational cost of model inference is equivalent to running one episode of simulation during the model training process. Therefore, the computational complexity of model inference is (\ref{eqn:A.2}) with $M$ = 1. \Halmos
\end{proof}

\section*{B. Illustrating Example for Continuous Model Fitness}

\begin{example} \label{example:1}
    We extract the demand-price data for a product from the Saudi Arabian market, covering the period from March 1, 2017, to March 16, 2018 (details on the data source are provided in Section 4.1.1). We select several common continuous demand models, including the linear demand, exponential demand, iso-elasticity demand, and the logit demand \citep{simchi2005logic}. Applying Ordinary/Nonlinear Least Squares (OLS/NLS) method, we obtain the optimal estimates of their model parameters, which refers to minimizing the sum of squared residuals. Subsequently, we fit the demand data using optimized demand models. The fitting results are shown in Figure \ref{fig:demand_price}.\par

From Figure \ref{fig:demand_price}, it can be seen that the fitness of these models with real data is rather poor. For the linear model, OLS regression analysis can be directly applied, and the details are shown in Table \ref{table:ols_linear}. For the exponential and iso-elasticity demand models, it is common to first take the natural logarithm and convert them into linear forms before applying OLS. The results are presented in Tables \ref{table:ols_exponential} and \ref{table:ols_iso}, respectively. Specifically, the $R^2$ of the linear model is only 0.115, indicating that the model can explain about 11.5\% of the variability in demand. The $R^2$ values for the exponential and iso-elasticity models are 0.106 and 0.116, respectively, which are similar to that of the linear model. The structure of the logit demand model prevents it from being estimated using OLS regression; however, we can still use NLS to fit the model and obtain an optimal $R^2$ value of 0.146. This is slightly higher but still indicates limited explanatory power, suggesting that the model is also not very ideal. \citet{schulte2020price} observe some instances in which a significant share of profits would be lost if the discrete nature of demand were not modeled explicitly. Thus, such continuous assumptions are not sufficient to apply in practice.  Apart from the reason of model assumption, this inadequacy also likely arises from the exclusion of significant influencing factors beyond the product's own price alone, such as competitor prices and market reference prices, which are taken into account in Eq.(\ref{eq:2}) of our study.
    \begin{figure} \label{demand_fit}
    \centering
    \caption{Demand-price relationship fitting under additive and multiplicative models}\includegraphics[width=0.45\textwidth]{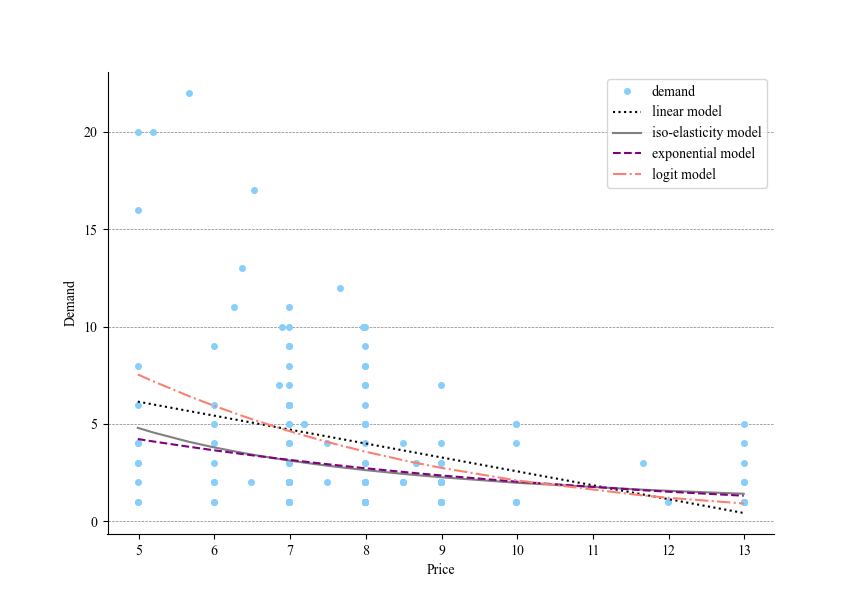}
    
    \label{fig:demand_price}
    \end{figure}

\begin{table}[h]
\centering
\caption{OLS Regression Results - Linear Model}
\label{table:ols_linear}
\small
\begin{tabular}{lcccccc}
\hline
                     & Coef.  & Std Err & $t$     & $P>|t|$ & [0.025 & 0.975] \\ \hline
\textbf{Constant}    & 9.7165 & 1.362   & 7.134 & 0.000            & 7.024  & 12.409 \\
\textbf{Price}       & -0.7151& 0.167   & -4.275& 0.000            & -1.046 & -0.384 \\ \hline
\multicolumn{7}{l}{\textit{Note:} Observations: 142, R-squared = 0.115, Adjusted R-squared = 0.109} \\
\multicolumn{7}{l}{F-statistic = 18.28, Prob(F-statistic): 3.51e-05,  AIC: 785.0, BIC: 790.9} \\

\end{tabular}
\end{table}

\begin{table}
\centering
\caption{OLS Regression Results - Exponential Model}
\label{table:ols_exponential}
\small
\begin{tabular}{lcccccc}
\hline
                     & Coef.  & Std Err & $t$     & $P>|t|$ & [0.025 & 0.975] \\ \hline
\textbf{Constant}    & 2.1693 & 0.292   & 7.421 & 0.000            & 1.591  & 2.747 \\
\textbf{Price}       &-0.1459 & 0.036 & -4.065 & 0.000 & -0.217 &     -0.075  \\ \hline
\multicolumn{7}{l}{\textit{Note:} Observations: 142, R-squared = 0.106, Adjusted R-squared = 0.099} \\
\multicolumn{7}{l}{F-statistic = 16.52, Prob(F-statistic): 7.98e-05,  AIC: 347.9, BIC: 353.8} \\

\end{tabular}
\end{table}

\begin{table}
\centering
\caption{OLS Regression Results - Iso-elasticity Model}
\label{table:ols_iso}
\small
\begin{tabular}{lcccccc}
\hline
                     & Coef.  & Std Err & $t$     & $P>|t|$ & [0.025 & 0.975] \\ \hline
\textbf{Constant}    & 3.6206 & 0.613   & 5.906 & 0.000            & 2.409  & 4.833 \\
\textbf{Price}       & -1.2764& 0.298   & -4.278& 0.000            & -1.866 & -0.686 \\ \hline
\multicolumn{7}{l}{\textit{Note:} Observations: 142, R-squared = 0.116, Adjusted R-squared = 0.109} \\
\multicolumn{7}{l}{F-statistic = 18.30, Prob(F-statistic): 3.48e-05, AIC: 346.3, BIC: 352.2} \\
\end{tabular}
\end{table}

\begin{figure}
\centering
\caption{Moments of demand}
\begin{subfigure}[b]{0.4\textwidth}
      \caption{\small The mean of the demand}\includegraphics[width=\textwidth]{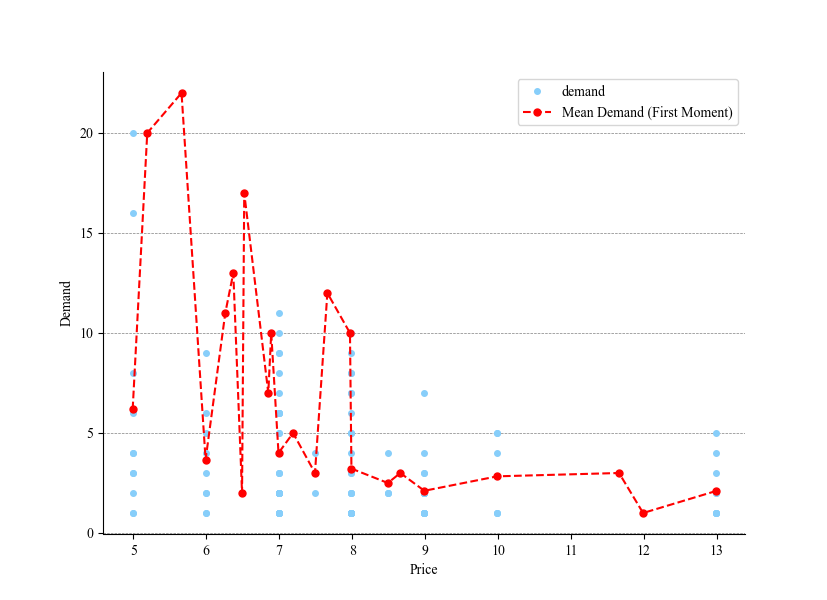}

   \label{fig:first_moment}
\end{subfigure}
\begin{subfigure}[b]{0.4\textwidth}
      \caption{\small The variance of the demand}\includegraphics[width=\textwidth]{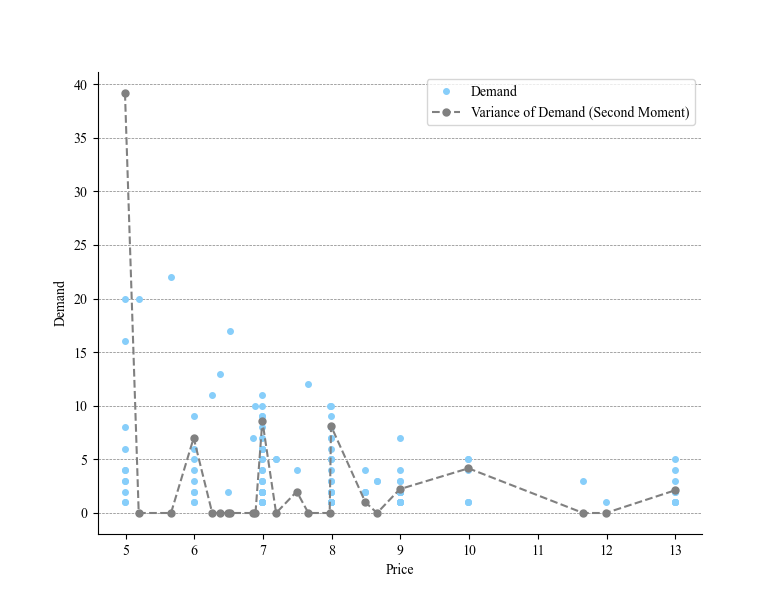}

   \label{fig:second_moment}
\end{subfigure}

\label{fig:moments}
\end{figure}

Figure \ref{fig:moments} describes the mean and variance of demand data. From Figure \ref{fig:first_moment}, we can observe that not all products' average demand decreases with price. From Figure \ref{fig:second_moment}, it is evident that the second moment of demand is exceptionally high at certain price points, indicating that market demand is volatile at some price levels. This variability is possibly influenced by multiple factors such as market competition, seasonal variations, and promotional activities, and is not solely dependent on price.
\end{example}

\section*{C. Single period Problem Results}

We consider a single period problem characterized by zero lead time, no startup inventory and no startup backlog, with parameters detailed in Table \ref{Parameters single}. The three-dimensional image of $F(p,x)$ is shown in Figure \ref{fig:non-concavity}. 

\begin{table}[H]
\centering 
\renewcommand{\arraystretch}{1.2}  
\caption{Parameters in single period problem}
\label{Parameters single}
\small
\begin{tabular}{c|cccccccccc}
\hline
Parameter & $\eta$ & $\Delta$ & $a$ & $l$ & $h$ & $b$ & $c$ & $x_0$ & $p_t$ & $x_t$ \\
\hline
Value & 800 & 0.5 & -4 & -0.01 & 4 & 10 & 5 & 0 & [0,80] & [0,20] \\
\hline
\end{tabular}
\end{table}

\begin{figure}[H]
  \centering
  \caption{Three-dimensional image of $F(p,x)$}
  \includegraphics[width=0.5\textwidth]{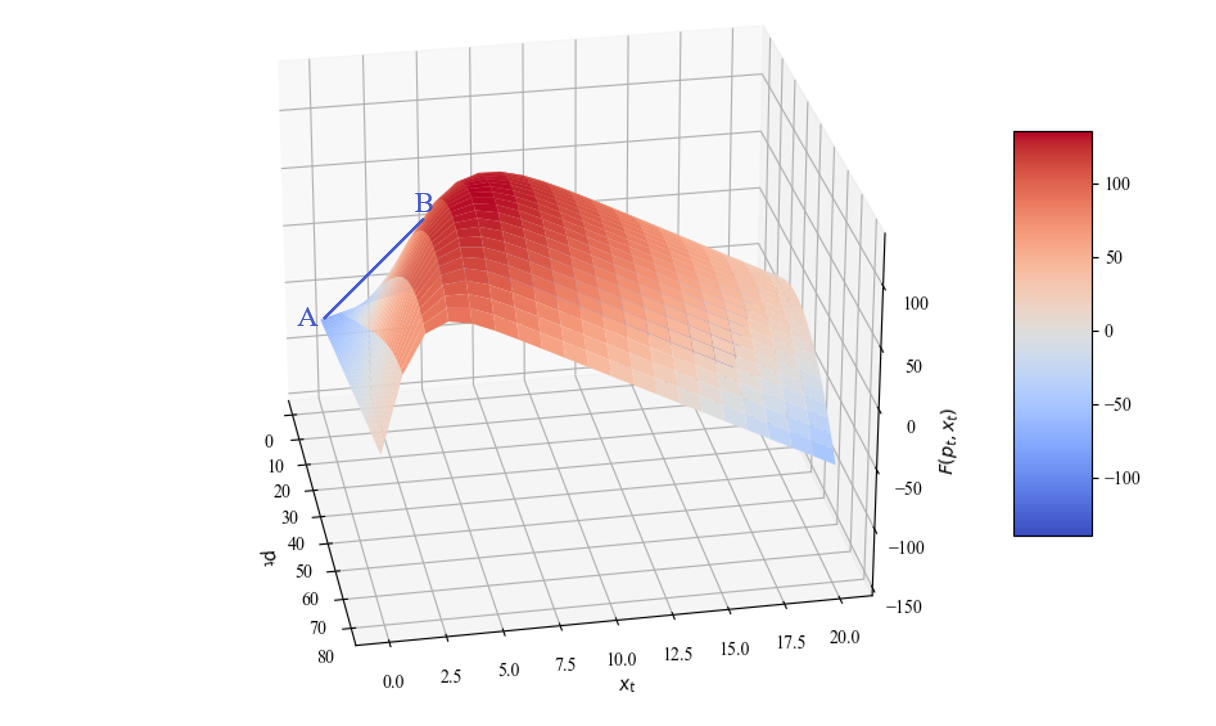} 

  \label{fig:non-concavity}
\end{figure}

\vspace{-0.5cm}
\begin{figure}[H]
  \centering
  \caption{Optimality of two-time-scale algorithm in single period problem}
  \includegraphics[width=0.55\textwidth]{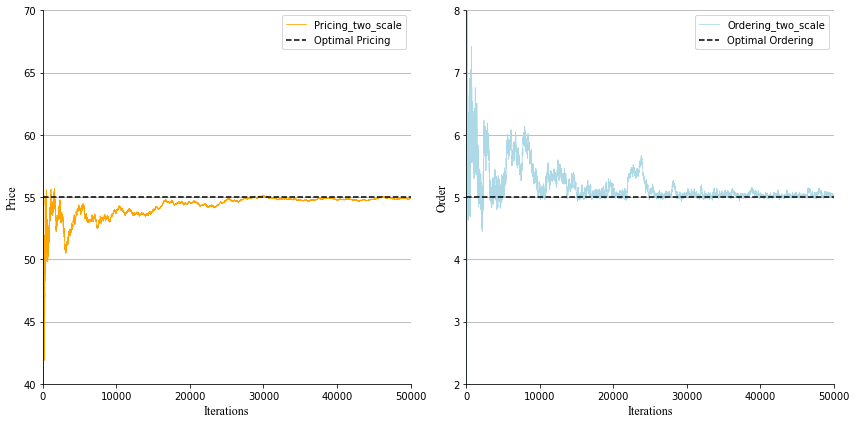} 
  
  \label{fig:two-scale-simple}
\end{figure}

\section*{D. LightGBM Training Process}
\begin{itemize}
    \item[1.] Data Preprocessing: (i) Address missing values, anomalies, and outliers.(ii) Employ Kalman filtering \citep{kalman1960new} for data smoothing. 
    \item[2.] Feature Engineering: (i)  Extract statistical features related to historical demand. (ii) Create cross-statistics for different contextual features such as seasons, weekends, weekdays, holidays, and their interactions on product demand. 
    \item[3.] Data Splitting: Divide the data into training and validation sets in chronological order, with the first 80\% of the data being the training set and the last 20\% being the validation set.
    \item[4.] Model Training and Parameter Optimization: First, train the LightGBM model on the training set and evaluate its performance on the validation set to find the optimal number of iterations. Next, retrain the model on the complete dataset with the optimal iterations. During training, LightGBM builds an ensemble of decision trees to minimize the loss function. Finally, optimize model parameters using grid search.
    \item[5.] Model performance is evaluated using the validation set with metrics such as Mean Square Error (MSE) and Mean Absolute Error (MAE).
\end{itemize}

\begin{figure}[H]
  \centering
  \caption{Feature importance ranking}
  \includegraphics[width=0.95\textwidth]{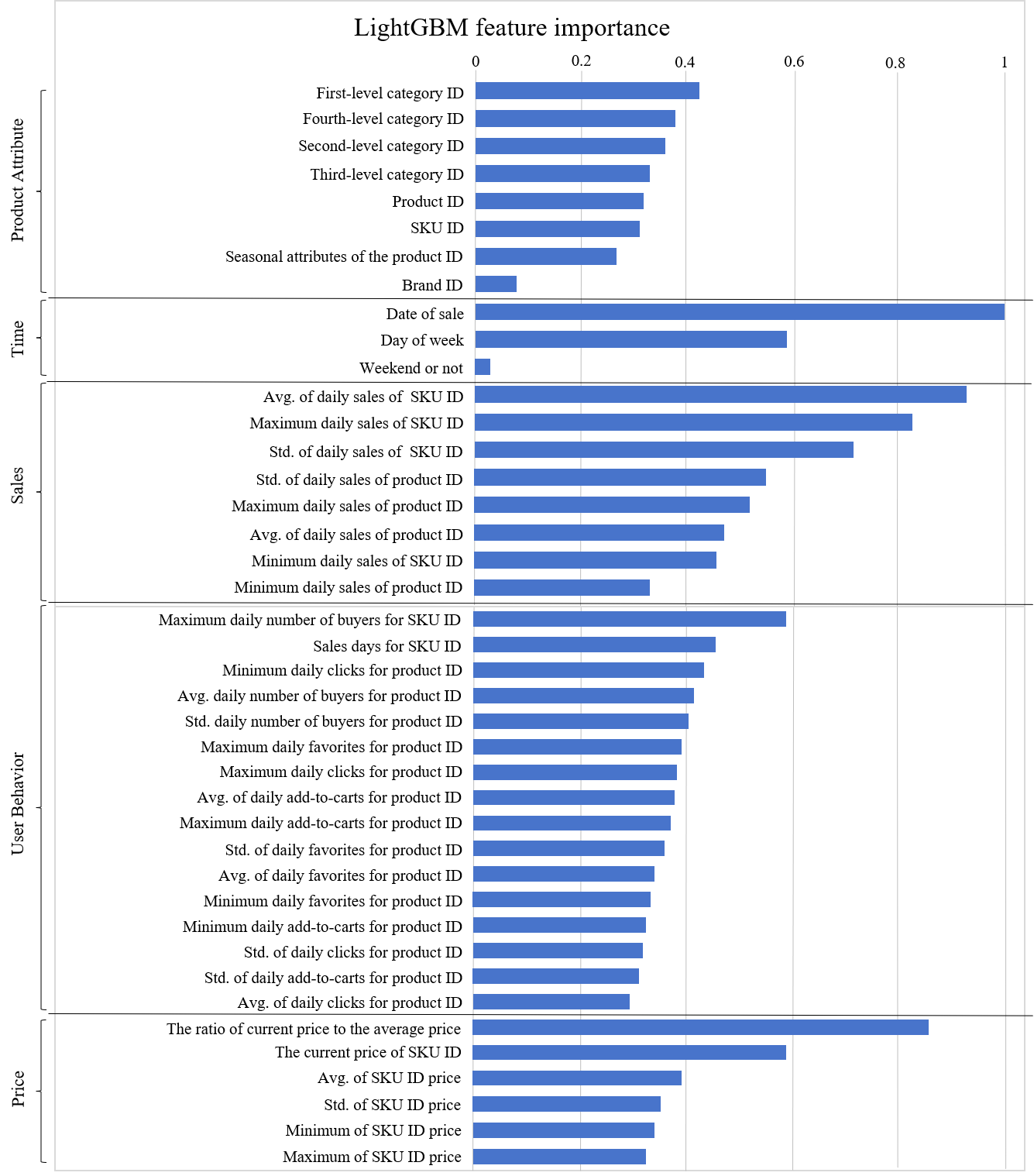} 

  \label{fig:feature}
\end{figure} 

\section*{E. Pseudo-code for FSDA}

\begin{algorithm}
\setlength{\columnwidth}{\linewidth}
\caption{Fast-slow Dual-agent Deep Reinforcement Learning  Algorithm} \label{Algo:1}
\begin{algorithmic}[1] \footnotesize
\Statex \textbf{Input:} Number of episodes $M$, periods per episode $T$,  discount factor $\gamma$, clip parameter $\epsilon$, entropy coefficient $\beta_{E}$, the ratio $k$ of two time scales $\alpha$ and $\beta$ 
\Statex \textbf{Initialize:} Actor networks  $(\pi_{\theta_0}^1,...,\pi_{\theta_0}^n)$, global critic network $v_{\phi_0}$ and replay buffer $\mathcal{B}$.
\For {$m=0,1,...,M-1$} 
    \State Collect a set of trajectories by running the joint policy $\bm{\pi}_{\theta_m}=(\pi_{\theta_m}^1,...,\pi_{\theta_m}^n)$.
    \State Push transitions \{$(s_t^i,a_t^i,s_{t+1}^i,r_t), i=1,...,n, t=1,...,T$\} into $\mathcal{B}$.
    \State Sample a random mini-batch of $B$ transitions from $\mathcal{B}$.
    \State Compute advantage function $A_t^i$ for actor $i$ in period $t$ based on global critic network $v_{\phi_m}$ with GAE. 
    \State Initialize sequential update factor $F_t^{1}=1, t=1,...,T$ for actor $1$
    \If{$m \mid k == 0$}
        \For {actor $i = 1,2$}
        \State Compute policy loss for actor $i$:
        $$L_p^i = -\frac{1}{BT} \sum_{b=1}^{B} \sum_{t=1}^{T} \min \Big( \frac{\pi_{\theta_{m+1}}^{i}(a_t^{i}|s_t^{i})}{\pi_{\theta_{m}}^{i}(a_t^{i}|s_t^{i})} F_t^{i} A_t^{i}, clip\Big( \frac{\pi_{\theta_{m+1}}^{i}(a_t^{i}|s_t^{i})}{\pi_{\theta_{m}}^{i}(a_t^{i}|s_t^{i})} , 1-\epsilon, 1+\epsilon \Big) F_t^{i} A_t^{i} \Big) $$ 
        \State Compute entropy loss for actor $i$: $$L_e^i= \beta_E \sum_{t=1}^{T} \pi_{\theta_m}^{i}(\cdot|s_t^{i}) log \big(\pi_{\theta_m}^{i}(\cdot|s_t^{i})\big)$$
        \State \parbox[t]{0.8\linewidth} {Update parameters $\theta_{m}$ of actor $\pi_{\theta_m}^{i}$ to $\theta_{m+1}$ by minimizing $L_p^i+L_e^i$ with the Adam optimizer \citep{kingma2014adam}} 
        \State Compute sequential update factor $F_t^{i+1}$ for actor $i+1$:
        $$F_t^{i+1} = F_t^{i} \frac{\pi_{\theta_{m+1}}^{i}(a_t^{i}|s_t^{i})}{\pi_{\theta_{m}}^{i}(a_t^{i}|s_t^{i})} $$
        \EndFor
    \Else
        \State Repeat lines 10-12 for actor $i = 2$
    \EndIf      
    \State Compute critic loss:
    $$L_c = \frac{1}{BT} \sum_{b=1}^{B} \sum_{t=1}^{T} (v_{\phi}(s_t)-r_t)^2$$
    \State \parbox[t]{0.8\linewidth} {Update parameters $\phi_m$ of critic $v_{\phi_m}$ to $\phi_{m+1}$ by minimizing $L_c$ with the Adam optimizer.}

\EndFor
\end{algorithmic}
\end{algorithm}
\end{document}